\documentclass{article}
\usepackage[left=2.25cm, right=2.25cm, top=3.00cm, bottom=2.00cm]{geometry}
\usepackage[utf8]{inputenc} 
\usepackage[T1]{fontenc}    
\usepackage{hyperref}       
\usepackage{url}            
\usepackage{booktabs}       
\usepackage{amsfonts}       
\usepackage{nicefrac}       
\usepackage{microtype}      

\usepackage{amsmath}
\usepackage{amsthm}
\usepackage{graphicx}
\usepackage{dsfont}
\usepackage{algorithm}
\usepackage{algorithmic}
\usepackage{soul,color}

\usepackage{subfigure}
\usepackage[font=small,labelfont=bf]{caption}

\DeclareMathOperator*{\argmin}{arg\,min}

\newtheorem{theorem}{Theorem}
\newtheorem{lemma}[theorem]{Lemma}
\newtheorem{definition}[theorem]{Definition}
\newtheorem{proposition}[theorem]{Proposition}
\newtheorem{corollary}[theorem]{Corollary}
\newtheorem{remark}[theorem]{Remark}
\newtheorem{fact}[theorem]{Fact}

\newcommand{\algonameCensUCB}{\textsc{PBM-UCB}}

\newcommand{\algonamePIE}{\textsc{PBM-PIE}}
\newcommand{\algonameTS}{\textsc{PBM-TS}}
\newcommand{\algonameRBA}{\textsc{RBA-KL-UCB}}

\allowdisplaybreaks

\title{Multiple-Play Bandits in the Position-Based Model}

\date{}

\author{Paul Lagrée\\
	Université Paris Sud, Université Paris Saclay \\ 
  \texttt{paul.lagree@u-psud.fr} \\
  Claire Vernade, Olivier Cappé \\
  Université Paris Saclay, Télécom ParisTech, CNRS\\
  \texttt{claire.vernade@telecom-paristech.fr} \\
  \texttt{cappe@enst.fr} \\
}

\begin{document}

\maketitle

\begin{abstract}
  Sequentially learning to place items in multi-position displays or lists is a task 
  that can be cast into the multiple-play
  semi-bandit setting. However, a major concern in this context is when the 
  system cannot decide whether the user feedback for each item is actually 
  exploitable. Indeed, much of the content may have been simply ignored by the
  user. The present work proposes to exploit available information 
  regarding the
  display position bias under the so-called Position-based click model
  (PBM). We first discuss how this model differs from the Cascade model and its
  variants considered in several recent works on multiple-play bandits. We then
  provide a novel regret lower bound for this model as well as computationally
  efficient algorithms that display good empirical and theoretical performance.
\end{abstract}

\section{Introduction}
\label{sec:introduction}

During their browsing experience, users are constantly provided -- without 
having asked for it -- with clickable content spread over web pages. 
While users interact on a website, they send clicks to the system for a very 
limited selection of the clickable content. Hence, they let every unclicked 
item with an 
equivocal answer: the system does not know whether the content was really 
deemed irrelevant or simply ignored. 
In contrast, in traditional multi-armed bandit (MAB) models,  the learner makes actions and observes at each round 
the reward corresponding to the chosen action. In the so-called multiple play semi-bandit 
setting, when users are presented with $L$ items, they are assumed to provide 
feedback for each of those items. 

Several variants of this basic setting have been considered in the bandit
literature. The necessity for the user to provide feedback for each item has
been called into question in the context of the so-called \emph{Cascade Model}
\cite{craswell2008experimental,kveton2015cascading,combes2015learning} and its
extensions such as the \emph{Dependent Click Model} (DCM) 
\cite{kveton2016dcm}.  Both models are particularly suited for search contexts,
where the user is assumed to be looking for something relative to a query.
Consequently, the learner expects explicit feedback: in the Cascade Model each
valid observation sequence must be either all zeros or terminated by a one, such
that no ambiguity is left on the evaluation of the presented items, while
multiple clicks are allowed in the DCM.

In the Cascade Model, the positions of the items are not taken into account in
the reward process because the learner is assumed to obtain a click as long as
the interesting item belongs to the list. Indeed, there are even clear
indications that the optimal strategy in a learning context consists in showing the most relevant
items at the end of the list in order to maximize the amount of observed
feedback \cite{kveton2015cascading} -- which is counter-intuitive in
recommendation tasks.

To overcome these limitations, \cite{combes2015learning} introduces weights --
to be defined by the learner -- that are attributed to positions in the list,
with a click on position $l \in \{1,\dots,L\}$ providing a reward $w_l$, where
the sequence $(w_l)_l$ is decreasing to enforce the ranking behavior. However,
no rule is given for setting the weights $(w_l)_l$ that control the order of importance
of the positions. The authors propose an algorithm based on KL-UCB
\cite{garivier11} and prove a lower bound on the regret as well as an
asymptotically optimal upper bound.

Another way to address the limitations of the Cascade Model is to consider the
DCM as in \cite{kveton2016dcm}. Here, examination probabilities $v_l$ are
introduced for each position $l$: conditionally on the event that the user
effectively scanned the list up to position $l$, he/she can choose to leave with
probability $v_l$ and in that case, the learner is aware of his/her departure.
This framework naturally induces the necessity to rank the items in the
optimal order.

All previous models assume that a portion of the recommendation list is
explicitly examined by the user and hence that the learning algorithm
eventually has access to rewards corresponding to the unbiased user's
evaluation of each item. In contrast, we propose to analyze multiple-play
bandits in the Position-based model (PBM) \cite{chuklin2015click}. In the PBM,
each position in the list is also endowed with a binary 
\emph{Examination variable}
\cite{craswell2008experimental,richardson2007predicting} which is equal to one
only when the user paid attention to the corresponding item. But this variable,
that is independent of the user's evaluation of the item, is not observable. It
allows to model situations where the user is not explicitly looking for
specific content, as in typical recommendation scenarios.

Compared to variants of the Cascade model, the PBM is challenging due to the
censoring induced by the examination variables: the learning algorithm observes
actual clicks but non-clicks are always ambiguous. Thus, combining observations
made at different positions becomes a non-trivial statistical task. Some
preliminary ideas on how to address this issue appear in the supplementary
material of \cite{komiyama2015optimal}. In this work, we provide a complete
statistical study of stochastic multiple-play bandits with semi-bandit feedback
in the PBM.

We introduce the model and notations in Section~\ref{sec:model} and provide the
lower bound on the regret in Section~\ref{sec:lower_bound}. In
Section~\ref{sec:algo}, we present two optimistic algorithms as well as
a theoretical analysis of their regret. In the last section dedicated to
experiments, those policies are compared to several benchmarks on both synthetic
and realistic data.


\section{Setting and Parameter Estimation}
\label{sec:model}
%

We consider the binary stochastic bandit model with $K$ Bernoulli-distributed
arms. The model parameters are the arm expectations
$\theta=(\theta_1,\theta_2,\ldots,\theta_K)$, which lie in $\Theta = (0,1)^K$.
We will denote by $\mathcal{B}(\theta)$ the Bernoulli
distribution with parameter $\theta$ and by
$d(p,q):=p\log(p/q)+(1-p)\log((1-p)/(1-q))$ the Kullback-Leibler divergence
from $\mathcal{B}(p)$ to $\mathcal{B}(q)$. At each round $t$, the learner
selects a list of $L$ arms --~referred to as an \emph{action}~-- chosen among
the $K$ arms which are indexed by $k \in \{1,\dots,K\}$. The set of actions is
denoted by $\mathcal{A}$ and thus contains $K!/(K-L)!$ ordered lists;
the action selected at time $t$ will be denoted
$A(t)=(A_1(t),\dots,A_L(t))$.

The PBM is characterized by examination parameters
$(\kappa_l)_{1\leq l\leq L}$, where $\kappa_l$ is the probability that the user
effectively observes the item in position $l$ \cite{chuklin2015click}. At round
$t$, the selection $A(t)$ is shown to the user and the learner observes the
complete feedback -- as in semi-bandit models -- but the observation at
position $l$, $Z_l(t)$, is \emph{censored} being the product of two independent Bernoulli
variables $Y_l(t)$ and $X_l(t)$, where $Y_l(t)\sim \mathcal{B}(\kappa_l)$ is
non null when the user considered the item in position $l$ -- which is unknown
to the learner -- and $X_l(t)\sim \mathcal{B}(\theta_{A_l(t)})$ represents the
actual user feedback to the item shown in position $l$. The learner receives a
reward $r_{A(t)} = \sum_{l=1}^L Z_l(t)$, where
$Z(t)=(X_1(t)Y_1(t),\ldots,X_L(t)Y_L(t))$ denotes the vector of censored
observations at step $t$.

In the following, we will assume, without loss of generality, that
$\theta_1>\dots>\theta_K$ and $\kappa_1 >\dots>\kappa_L>0$, in order to simplify
the notations. The fact that the sequences $(\theta_l)_l$ and $(\kappa_l)_l$ are
decreasing implies
that the optimal list is $a^* = (1, \ldots, L)$. Denoting by
$R(T) = \sum_{t=1}^T r_{a^*} - r_{A(t)}$ the regret incurred by the learner
up to time $T$, one has
\begin{equation}
\mathds{E}[R(T)] = \sum_{t=1}^{T}\sum_{l=1}^{L}\kappa_l
(\theta_{a^*_l} - \mathds{E}[\theta_{A_l(t)}])
= \sum_{a\in \mathcal{A}}
\left(\mu^*-\mu_a \right) \mathds{E}[N_a(T)]= \sum_{a\in \mathcal{A}}
\Delta_a \mathds{E}[N_a(T)], 
\label{eq:exp_regret}
\end{equation}
where $\mu_a = \sum_{l=1}^L \kappa_l \theta_{a_l}$ is the expected reward of
action $a$, $\mu^* =\mu_{a^*}$ is the best possible reward in average,
$\Delta_a=\mu^* -\mu_{a}$ the expected gap to optimality, and, $N_a(T) =
\sum_{t=1}^T \mathds{1} \{A(t)=a\}$ is the number of times action $a$ has been 
chosen
up to time $T$.

In the following, we assume that the examination parameters $(\kappa_l)_{1\leq
l \leq L}$ are known to the learner. These can be estimated from
historical data \cite{chuklin2015click}, using, for instance, the EM algorithm
\cite{dempster1977maximum} (see also Section~\ref{sec:experiments}). In most
scenarios, it is realistic to assume that
the content (e.g., ads in on-line advertising) is changing much more frequently
than the layout (web page design for instance) making it possible to have a
good knowledge of the click-through biases associated with the display
positions.

The main statistical challenge associated with the PBM is that
one needs to obtain estimates and confidence bounds for the components
$\theta_k$ of $\theta$ from the available
$\mathcal{B}(\kappa_l\theta_k)$-distributed draws corresponding to occurrences
of arm $k$ at various positions $l=1,\dots,L$ in the list. To this aim, we
define the following statistics: $S_{k,l}(t)=\sum_{s=1}^{t-1} Z_l(s) \mathds{1}
\{A_l(s)=k\}$, $S_{k}(t)=\sum_{l=1}^L S_{k,l}(t)$, $N_{k,l}(t)=\sum_{s=1}^{t-1}
\mathds{1} \{A_l(s)=k\}$, $N_k(t)=\sum_{l=1}^L N_{k,l}(t)$. We further require 
bias-corrected versions of the counts $\tilde{N}_{k,l}(t)=\sum_{s=1}^{t-1}
\kappa_l \mathds{1} \{A_l(s)=k\}$ and $\tilde{N}_{k}(t)=\sum_{l=1}^L
\tilde{N}_{k,l}(t)$. 

A time $t$, and \emph{conditionally on the past actions $A(1)$ up to
$A(t-1)$}, the Fisher information for
$\theta_k$ is given by $I(\theta_k) = \sum_{l=1}^L N_{k,l}(t) \kappa_l /
(\theta_k (1-\kappa_l \theta_k))$ (see Appendix~\ref{sec:appendix-estim}). We cannot however estimate $\theta_k$ 
using the maximum likelihood estimator since it has no closed form expression.
Interestingly though, the simple pooled linear estimator
\begin{equation}
    \hat{\theta}_k(t) = S_{k}(t)/\tilde{N}_{k}(t),
    \label{eq:theta_estimator}
\end{equation}
considered in the supplementary material to \cite{komiyama2015optimal}, is
unbiased and has a (conditional) variance of
$\upsilon(\theta_k) = (\sum_{l=1}^L N_{k,l}(t)\kappa_l
\theta_k(1-\kappa_l\theta_k))/(\sum_{l=1}^L N_{k,l}(t) \kappa_l)^2$, which is
close to optimal given the Cramér-Rao lower bound. Indeed, 
$\upsilon(\theta_k)I(\theta_k)$ is recognized as a ratio of a
weighted arithmetic mean to the corresponding weighted harmonic mean, which is known to be larger than
one,
but is upper
bounded by $1/(1-\theta_k)$, irrespectively of the values of the
$\kappa_l$'s. Hence, if, for instance, we can assume that all $\theta_k$'s are
smaller than one half, the loss with respect to the best unbiased estimator is
no more than a factor of two for the variance.
Note that
despite its simplicity, $\hat{\theta}_k(t)$ cannot be written as a simple sum
of conditionally independent increments divided by the number of terms and will
thus require specific concentration results.

It can be checked that when $\theta_k$ gets very close to one,
$\hat{\theta}_k(t)$ is no longer close to optimal. This observation also has a
Bayesian counterpart that will be discussed in
Section~\ref{sec:experiments}. Nevertheless, it is always preferable to the
``position-debiased'' estimator $(\sum_{l=1}^L S_{k,l}(t)/\kappa_l)/N_{k,l}(t)$
which gets very unreliable as soon as one of the $\kappa_l$'s gets
very small.


%

\section{Lower Bound on the Regret}
\label{sec:lower_bound}
%


In this section, we consider the fundamental asymptotic limits of learning 
performance for online algorithms under the PBM. These cannot be deduced from 
earlier general results, such as those of
\cite{graves1997asymptotically,combes2015combinatorial}, due to the censoring
in the feedback associated to each action. We 
detail a simple and general proof scheme -- using the results of 
\cite{kaufmann2015complexity} -- that applies to the PBM, as well as to more 
general models.

Lower bounds on the regret rely on changes of 
measure: the question is how much can we mistake the true parameters of the 
problem for others, when observing successive arms? With this in mind,  we will 
subscript all expectations and probabilities by the parameter value and 
indicate explicitly that the quantities $\mu_a, a^*, \mu^*, \Delta_a$, 
introduced in Section~\ref{sec:model}, also depend on the parameter. For ease 
of notation, we will still assume that $\theta$ is such that $a^*(\theta) = (1, 
\dots, L$).

\subsection{Existing results for multiple-play bandit problems}
\label{subsec:lbexisting}

Lower bounds on the regret will be proved for 
\emph{uniformly efficient} algorithms, in the sense of 
\cite{lai1985asymptotically}:

\begin{definition}\label{def:uniformly_eff}
	An algorithm is said to be \emph{uniformly efficient} if for any 
    bandit model parameterized by $\theta$ and for all $\alpha\in(0,1]$, its
	expected regret after $T$ rounds is such that 
	$\mathds{E}_\theta R(T) = o(T^\alpha)$.
\end{definition}

For the multiple-play MAB, \cite{anantharam1987asymptotically} obtained the
following bound

\begin{align}
    \liminf_{T\rightarrow \infty}\frac{\mathds{E}_\theta R(T)}{\log(T)} \geq 
    \sum_{k = L+1}^K\frac{\theta_L-\theta_k}{d(\theta_k,\theta_L)}.
    \label{eq:lb_original}
\end{align}

For the ``learning to rank'' problem where rewards follow the weighted Cascade 
Model with decreasing weights $\left(w_l\right)_{l=1,\ldots,L}$,
\cite{combes2015learning} derived the following bound
 
\[\liminf_{T\rightarrow \infty} \frac{\mathds{E}_\theta R(T)}{\log T} \geq 
w_L \sum_{k=L+1}^{K} 
\frac{\theta_L-\theta_k}{d(\theta_k,\theta_L)}.
\]

Perhaps surprisingly, this lower bound does not show any additional term 
corresponding to the complexity of ranking the $L$ optimal arms. Indeed, the 
errors are still asymptotically dominated by the need to discriminate 
irrelevant arms $(\theta_k)_{k>L}$ from the worst of the relevant arms,
that is, $\theta_L$.

\subsection{Lower bound step by step}


\paragraph*{Step 1: Computing the expected log-likelihood ratio.}

Denoting by  $\mathcal{F}_{s-1}$ the $\sigma$-algebra generated by 
the past actions and observations, we define the log-likelihood ratio for the 
two values $\theta$ and $\lambda$ of the parameters by
\begin{equation}
  \ell(t) := \sum_{s=1}^t \log \frac{p(Z(s); \theta~|~\mathcal{F}_{s-1})}
{p(Z(s); \lambda~|~\mathcal{F}_{s-1})}.
  \label{eq:log-likelihood-ratio}
\end{equation}

\begin{lemma}\label{lem:loglikelihood}
	For each position $l$ and each item $k$, define the local amount of 
	information by
	$$I_l(\theta_k,\lambda_k):= 
	\mathds{E}_{\theta}\left[\left.\log\frac{p(Z_l(t);\theta)}
	{p(Z_l(t); \lambda)}\right| A_l(t)=k\right], $$
	and its cumulated sum over the $L$ positions by  		
	$I_a(\theta,\lambda):=\sum_{l=1}^{L} \sum_{k=1}^K
    \mathds{1}\{a_l = k\}I_l(\theta_k,\lambda_k)$. 
	The expected log-likelihood ratio is given by
	\begin{equation}
	\label{eq:llr_exp}
    \mathds{E}_{\theta}[\ell(t)] = \sum_{a\in \mathcal{A}}
	I_a(\theta,\lambda)\mathds{E}_\theta[N_a(t)].
	\end{equation}
\end{lemma}

The next proposition is adapted from Theorem~17 in Appendix B of 
\cite{kaufmann2015complexity} and provides 
a lower bound on the expected log-likelihood ratio.

\begin{proposition}\label{prop:constraintLB}
	Let $B(\theta):= \{ \lambda \in \Theta \left| 
	\forall l\leq L, \theta_l=\lambda_l \text{ and }  
	\mu^*(\theta)<\mu^*(\lambda) \right. \}$ 
	be the set of changes of measure that improve over $\theta$ 
	without modifying the optimal arms.
	Assuming that the expectation of the log-likelihood ratio may be written as 
	in~\eqref{eq:llr_exp}, for any uniformly efficient algorithm one has
	$$ \forall \lambda \in B(\theta), \hspace{0.5cm} 
	\liminf_{T\rightarrow \infty} \dfrac{\sum_{a\in 
			\mathcal{A}}I_a(\theta,\lambda)\mathds{E}_\theta[N_a(T)]}
	{\log (T)} \geq 1.
	$$
\end{proposition}

\paragraph*{Step 2: Variational form of the lower bound.}
We are now ready to obtain the lower bound in a form similar to that originally 
given by \cite{graves1997asymptotically}.

\begin{theorem}\label{th:LB_opt}
	The expected regret of any uniformly 
	efficient algorithm satisfies
        \begin{equation*}	
        \liminf_{T\rightarrow \infty} \frac{\mathds{E}_\theta R(T)}{\log T} 
	\geq  
	f(\theta) \, ,
	\quad \text{where } f(\theta) = \inf_{c\succeq 0}\sum_{a\in \mathcal{A}} 
	\Delta_{a}(\theta) c_a \, , 
	\quad \text{s.t. }  \inf_{\lambda \in B(\theta)} \sum_{a\in 
		\mathcal{A}} I_a(\theta,\lambda)c_a \geq 1.
       \end{equation*}	
\end{theorem}

Theorem~\ref{th:LB_opt} is a straightforward consequence of
Proposition~\ref{prop:constraintLB}, combined with the expression of the
expected regret given in~(\ref{eq:exp_regret}). The vector $c\in
\mathbb{R}_+^{|\mathcal{A}|}$, that satisfies the inequality
$\sum_{a\in \mathcal{A}} I_a(\theta,\lambda)c_a \geq 1$, represents the
feasible values of $\mathds{E}_\theta[N_a(T)]/\log(T)$. 

\paragraph*{Step 3: Relaxing the constraints.}
The bounds mentioned in Section~\ref{subsec:lbexisting} may
be recovered from Theorem~\ref{th:LB_opt} by considering only the changes of
measure that affect a single suboptimal arm.

\begin{corollary}\label{cor:LB_relaxed}
	\begin{equation*}
	f(\theta)\geq \inf_{c\succeq 0}\sum_{a\in \mathcal{A}} 
	\Delta_a(\theta) c_a \, , 
	\quad \text{s.t. } \sum_{a\in \mathcal{A}} \sum_{l=1}^L 
	\mathds{1}\{a_l=k\} I_l(\theta_k,\theta_L)c_a \geq 1 \, , \quad \forall k
    \in \{L+1, \ldots, K\}.
	\end{equation*} 
\end{corollary} 

Corollary~\ref{cor:LB_relaxed} is obtained by restricting the constraint set 
$B(\theta)$  of Theorem~\ref{th:LB_opt} to $\cup_{k=L+1}^K B_k(\theta)$, where
$
B_k(\theta):= \left\lbrace \lambda \in \Theta | 
\forall j\neq k ,\theta_j=\lambda_j \text{ and } \mu^*(\theta)<\mu^*(\lambda) 
\right\rbrace. 
$ 

\subsection{Lower bound for the PBM}

\begin{theorem}\label{th:LB_apm}
	For the PBM, the following lower bound holds for any uniformly efficient 
	algorithm:
    \begin{equation}
	    \liminf_{T\rightarrow \infty} \frac{\mathds{E}_\theta R(T)}{\log T} \geq
	    \sum_{k=L+1}^K \min_{l \in \{1,\ldots,L\}} 
        \frac{\Delta_{v_{k,l}}(\theta)}{d(\kappa_l\theta_k,\kappa_l\theta_L)},
        \label{eq:LB_apm}
    \end{equation}
    where $v_{k,l} := (1,\dots,l-1,k,l,\dots,L-1)$.
\end{theorem}
\begin{proof}
	First, note that
	for the PBM one has $I_l(\theta_k,\lambda_k) =
	d(\kappa_l\theta_k,\kappa_l\lambda_k)$. To get the expression given in
	Theorem~\ref{th:LB_apm} from
	Corollary~\ref{cor:LB_relaxed}, we proceed as in \cite{combes2015learning} 
	showing
	that the optimal coefficients $(c_a)_{a\in\mathcal{A}}$ can be non-zero 
	only for the $K-L$ 	actions that put the suboptimal arm $k$ in the position $l$
    that reaches the minimum of
    $\Delta_{v_{k,l}}(\theta)/d(\kappa_l\theta_k,\kappa_l\theta_L)$. 
	Nevertheless, this position does not always coincide with $L$, the end of
    the displayed list, contrary to the case of \cite{combes2015learning} (see Appendix~\ref{sec:proof-lower-bound} for details).
\end{proof}

The discrete minimization that appears in the r.h.s.\ of Theorem~\ref{th:LB_apm}
corresponds to a fundamental trade-off in the PBM. When trying to
discriminate a suboptimal arm $k$ from the $L$ optimal ones, it is desirable to
put it higher in the list to obtain more information, as
$d(\kappa_l\theta_k,\kappa_l\theta_L)$ is an increasing function of
$\kappa_l$. On the other hand, the gap $\Delta_{v_{k,l}}(\theta)$ is also
increasing as $l$ gets closer to the top of the list. The fact that 
$d(\kappa_l\theta_k,\kappa_l\theta_L)$ is not linear in $\kappa_l$ (it is a 
strictly convex function of $\kappa_l$) renders the trade-off
non trivial. It is easily checked that when $(\theta_1 - \theta_L)$ is very
small, i.e. when all optimal arms are equivalent, the optimal exploratory 
position is $l=1$. In contrast, it is equal to $L$ when the gap
$(\theta_L -\theta_{L+1})$ becomes very small. Note that by
using  that for any suboptimal $a\in \mathcal{A}$, $\Delta_a(\theta)\geq 
\sum_{k=L+1}^K \sum_{l=1}^{L} \mathds{1}\{a_l = k\}\kappa_l (\theta_L-\theta_k)$,
one can lower bound the r.h.s. of Theorem~\ref{th:LB_apm} by
$\kappa_L \sum_{k=L+1}^K  (\theta_L -
\theta_k)/d(\kappa_L\theta_k,\kappa_L\theta_L)$, which is not tight in general.

\begin{remark}
  In the uncensored version of the PBM --~i.e., if the $Y_l(t)$ were
  observed~--, the expression of $I_a(\theta,\lambda)$ is simpler: it is equal
  to $\sum_{l=1}^{L}\sum_{k=1}^{K}\mathds{1}\{A_l(t)=k\}\kappa_l
  d(\theta_k,\lambda_k)$ and leads to a lower bound that coincides with
  (\ref{eq:lb_original}). The uncensored PBM is actually statistically very
  close to the weighted Cascade model and can be addressed by algorithms that do
  not assume knowledge of the $(\kappa_l)_l$ but only of their ordering.
\end{remark}


\section{Algorithms}
\label{sec:algo}

In this section we introduce two algorithms for the PBM. The first one uses the
CUCB strategy of \cite{chen2013combinatorial} and requires an
simple upper confidence bound for $\theta_k$ based on the estimator
$\hat{\theta}_k(t)$ defined in~\eqref{eq:theta_estimator}. The second algorithm
is based on the \emph{Parsimonious Item Exploration} --~PIE(L)~-- scheme
proposed in \cite{combes2015learning} and aims at reaching asymptotically
optimal performance. For this second algorithm, termed \algonamePIE, it is
also necessary to use a multi-position analog of the well-known KL-UCB index
\cite{garivier11} that is inspired by a result of \cite{magureanu2014lipschitz}. The  analysis of  \algonamePIE\ provided below confirms the relevance of the lower bound derived in Section~\ref{sec:lower_bound}. 

\paragraph*{\algonameCensUCB} 
The first algorithm simply consists in sorting optimistic indices in decreasing
order and pulling the corresponding first $L$ arms
\cite{chen2013combinatorial}. To derive the expression of the required
``exploration bonus'' we use an upper confidence for $\hat{\theta}_k(t)$ based on Hoeffding's inequality:
$$
U^{UCB}_k(t,\delta) = \frac{S_k(t)}{\tilde{N}_k(t)} + 
\sqrt{\frac{N_k(t)}{\tilde{N}_k(t)}}\sqrt{\frac{\delta}{2\tilde{N}_k(t)}},
$$
for which a coverage bound is given by the next proposition, proven in 
Appendix~\ref{sec:concentrationUCB}. 
\begin{proposition}\label{prop:boundCensUCB}
	Let $k$ be any arm in $\{1,\ldots,K\}$, then for any $\delta > 0$,
	$$
	\mathds{P}\left(U^{UCB}_k(t,\delta) \leq \theta_k \right) \leq
	e \delta \log(t) e^{-\delta}.
	$$
\end{proposition}


Following the ideas of \cite{combes2015combinatorial}, it is possible to obtain
a logarithmic regret upper bound for this algorithm. The proof is given in
Appendix~\ref{sec:proofUCB}.

\begin{theorem}\label{th:pbm-ucb}
  Let $C(\kappa) = \min_{1\leq l\leq L}
  [(\sum_{j=1}^L\kappa_j)^2/l+(\sum_{j=1}^l\kappa_j)^2]/\kappa_L^2$ and $\Delta
  = \min_{a \in \sigma(a^*) \setminus a^*} \Delta_a$, where $\sigma(a^*)$
  denotes the permutations of the optimal action.  Using \algonameCensUCB~ with
  $\delta=(1+\epsilon)\log(t)$ for some $\epsilon>0$, there exists a constant
  $C_0(\epsilon)$ independent from the model parameters such that the regret of
  \algonameCensUCB\ is bounded from above by
    $$ \mathds{E}[R(T)] \leq C_0(\epsilon) + 16(1+\epsilon) C(\kappa)\log T \left( \frac{L}{\Delta} +\sum_{k\notin 
		a^*}\frac{1}{\kappa_L(\theta_L-\theta_k)}\right).$$
\end{theorem}

The presence of the term $L/\Delta$ in the above expression is
attributable to limitations of the mathematical analysis. On the other hand,
the absence of the KL-divergence terms appearing in the lower
bound~\eqref{eq:LB_apm} is due to the use of an upper confidence bound based on
Hoeffding's inequality.

\paragraph*{\algonamePIE}
We adapt the PIE($l$) algorithm introduced by \cite{combes2015learning} for the
Cascade Model to the PBM in Algorithm~\ref{alg:PIE} below. At each round,
the learner potentially explores at position $L$ with probability $1/2$ using the 
following
upper-confidence bound for each arm $k$
\begin{equation}
    U_k(t,\delta) = \sup_{q\in [\theta_k^{\min}(t),1]} \left 
    \{q\left\vert\sum_{l=1}^L N_{k,l}(t)
    d\left(\frac{S_{k,l}(t)}{N_{k,l}(t)},\kappa_l q\right)\leq \delta \right. 
    \right\},
   \label{eq:pbm-pie_index}
\end{equation}
where $\theta_k^{\min}(t)$ is the minimum of the convex function 
$\Phi:q\mapsto\sum_{l=1}^L N_{k,l}(t) d(S_{k,l}(t)/N_{k,l}(t), \kappa_l q)$.
In other positions, $l=1,\dots,L-1$, PBM-PIE selects the arms with the largest
estimates $\hat{\theta}_k(t)$. The resulting algorithm is presented as
Algorithm~\ref{alg:PIE} below, denoting by $\mathcal{L}(t)$ the $L$-largest
empirical estimates, referred to as the ``leaders'' at round $t$.

\begin{algorithm}[hbt]
	\caption{ -- \algonamePIE}
	\begin{algorithmic}\small
		\REQUIRE{$K$, $L$, observation probabilities $\kappa$, $\epsilon>0$}
		\STATE{Initialization: first $K$ rounds, play each arm at every 
		position}
		\FOR{$t = K+1,\ldots,T$}
            \STATE Compute $\hat{\theta}_k(t)$ for all $k$
            \STATE $\mathcal{L}(t) \gets$ top-$L$ ordered arms by decreasing
                   $\hat{\theta}_k(t)$
            \STATE $A_l(t) \gets \mathcal{L}_l(t)$ for each position $l < L$
            \STATE $\mathcal{B}(t) \gets \{k | k \notin \mathcal{L}(t), 
            U_k(t,(1+\epsilon)\log(T))
                \geq \hat{\theta}_{\mathcal{L}_L(t)}(t)$
            \IF{$\mathcal{B}(t) = \emptyset$}
                \STATE $A_L(t) \gets \mathcal{L}_L(t)$
            \ELSE
            \STATE With probability $1/2$, select $A_L(t)$ uniformly at random 
            from
                $\mathcal{B}(t)$, else $A_L(t) \gets \mathcal{L}_L(t)$
            \ENDIF
            \STATE Play action $A(t)$ and observe feedback $Z(t)$; Update $N_{k,l}(t+1)$ and $S_{k,l}(t+1)$.
		\ENDFOR
	\end{algorithmic}
	\label{alg:PIE}
\end{algorithm} 

The $U_k(t,\delta)$ index defined in~\eqref{eq:pbm-pie_index} aggregates
observations from all positions --~as in \algonameCensUCB~-- but
allows to build tighter confidence regions as shown by the next proposition proved in Appendix~\ref{sec:proofUBPIE}.

\begin{proposition}\label{prop:PIEcontrol}
 For all $\delta \geq L+1$,
	$$
	\mathds{P}\left(U_k(t,\delta) < \theta_k\right) \leq 
	e^{L+1} \left(\frac{\left\lceil\delta \log(t)\right\rceil
		\delta}{L}\right)^L e^{-\delta} .
	$$
\end{proposition}

We may now state the main result of this section that provides an
upper bound on the regret of \algonamePIE.

\begin{theorem}\label{th:pbm-pie}
	Using \algonamePIE\ with $\delta = (1+\epsilon)\log(t)$ and $\epsilon >
	0$, for any $\eta<\min_{k<K} (\theta_k-\theta_{k+1})/2$, there exist
	problem-dependent constants $C_1(\eta)$, $C_2(\epsilon,\eta),C_3(\epsilon)$ 
	and $\beta(\epsilon,\eta)$ such that
    \[
        \mathds{E}[R(T)] \leq (1+\epsilon)^2 \log(T)\sum_{k=L+1}^K \frac{\kappa_L
        (\theta_L - \theta_k)}{d(\kappa_L\theta_k,\kappa_L(\theta_L-\eta))} + C_1(\eta)
        + \frac{C_2(\epsilon,\eta)}{T^{\beta(\epsilon,\eta)}}+C_3(\epsilon).
    \]
\end{theorem}

The proof of this result is provided in
Appendix~\ref{sec:proofUBPIE}. Comparing to the expression
in~(\ref{eq:LB_apm}), Theorem~\ref{th:pbm-pie} shows that PBM-PIE reaches
asymptotically optimal performance when the optimal exploring position is
indeed located at index $L$. In other case, there is a gap that is
caused by the fact the exploring position is fixed beforehand and not adapted
from the data.



We conclude this section by a quick description of two other algorithms that will be used in the experimental section to benchmark our results.

\paragraph*{Ranked Bandits (\algonameRBA)} The state-of-the-art algorithm for
the sequential ``learning to rank'' problem was proposed by
\cite{radlinski2008learning}. It runs one bandit algorithm per position, each
one being entitled to choose the best suited arm at its rank. The underlying
bandit algorithm that runs in each position is left to the choice of the user,
the better the policy the lower the regret can be. If the bandit algorithm at
position $l$ selects an arm already chosen at a higher position, it receives a
reward of zero.  Consequently, the bandit algorithm operating at position $l$
tends to focus on the estimation of $l$-th best arm. In the next section, we use
as benchmark the Ranked Bandits strategy using the KL-UCB algorithm
\cite{garivier11} as the per-position bandit.

\paragraph*{\algonameTS}
The observations $Z_l(t)$ are censored
Bernoulli which results in a posterior that does not belong to a standard 
family of distribution. \cite{komiyama2015optimal}
suggest a version of Thompson Sampling called  ``Bias Corrected Multiple Play 
TS'' (or BC-MP-TS) that approximates the true
posterior by a Beta distribution. We observed in experiments that for
parameter values close to one, this algorithm does not explore enough. In 
Figure~\ref{fig:d-ts}, we show this phenomenon for $\theta = (0.95, 0.85, 0.75, 
0.65, 0.55)$. The true posterior 
for the parameter
$\theta_k$ at time $t$ may be written as a product of truncated scaled beta 
distributions
\[\pi_t(\theta_k)\propto \prod_l 
\theta_k^{\alpha_{k,l}(t)}(1-\kappa_l\theta_k)^{\beta_{k,l}(t)},\]
where $\alpha_{k,l}(t) = S_{k,l}(t)$ and $\beta_{k,l}(t) = N_{k,l}(t)-
S_{k,l}(t)$. To draw from this exact posterior, we use rejection sampling with
proposal distribution 
$\operatorname{Beta}(\alpha_{k,m}(t),\beta_{k,m}(t))/\kappa_{m}$,
where $m = \arg\max_{1\leq l \leq L} (\alpha_{k,l}(t)+\beta_{k,l}(t))$.


\section{Experiments}
\label{sec:experiments}


\begin{figure}
	\centering
	\subfigure[\small Average regret of \algonameTS~and BC-MP-TS compared for 
	high parameters. Shaded areas:
    first and last deciles.
	\label{fig:d-ts}]{\includegraphics[height=5cm]
		{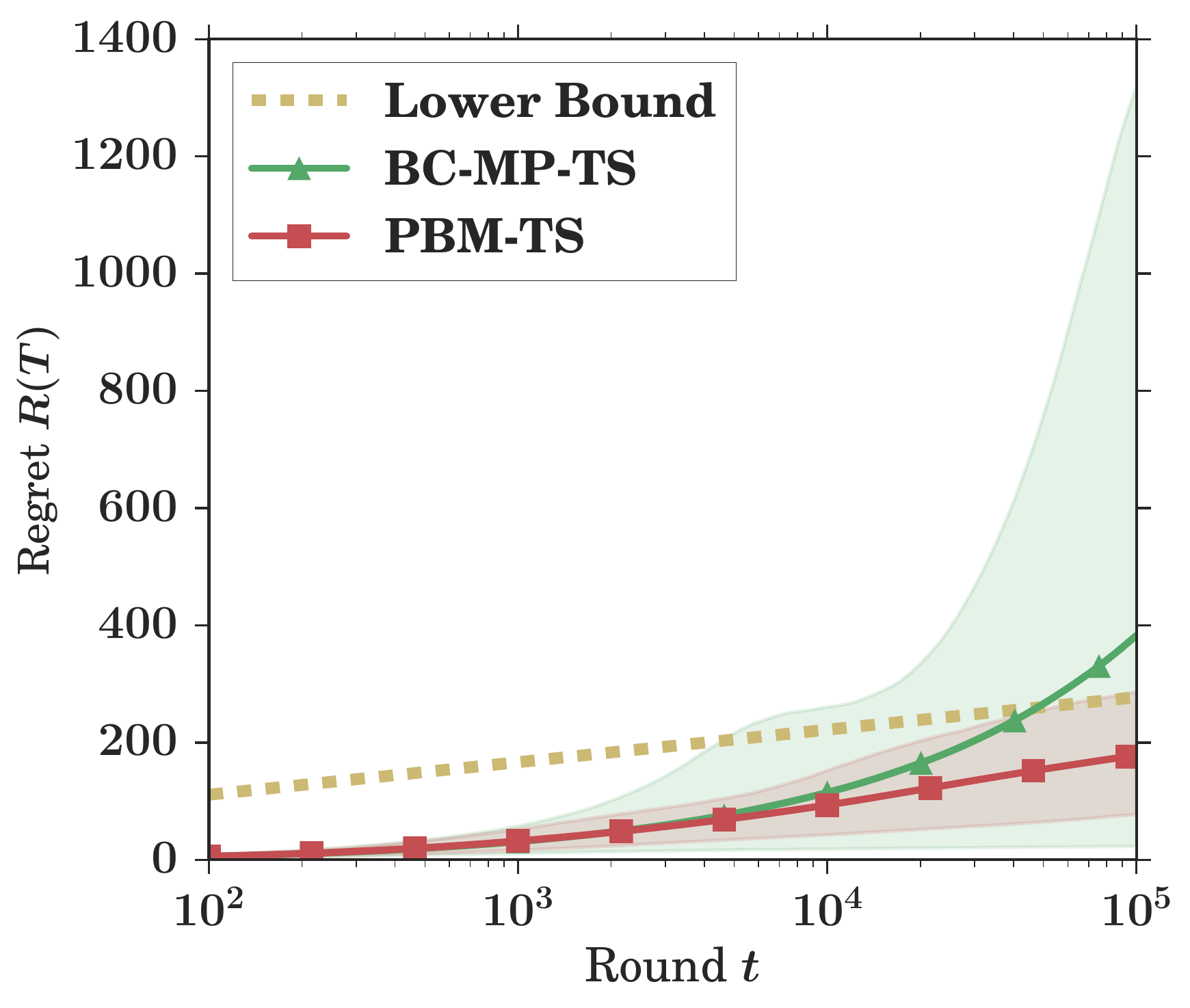}}~\quad
	\subfigure[Average regret of various algorithms on synthetic 
data under the 
PBM.\label{fig:DVA_simu}]{\includegraphics[height=5cm]{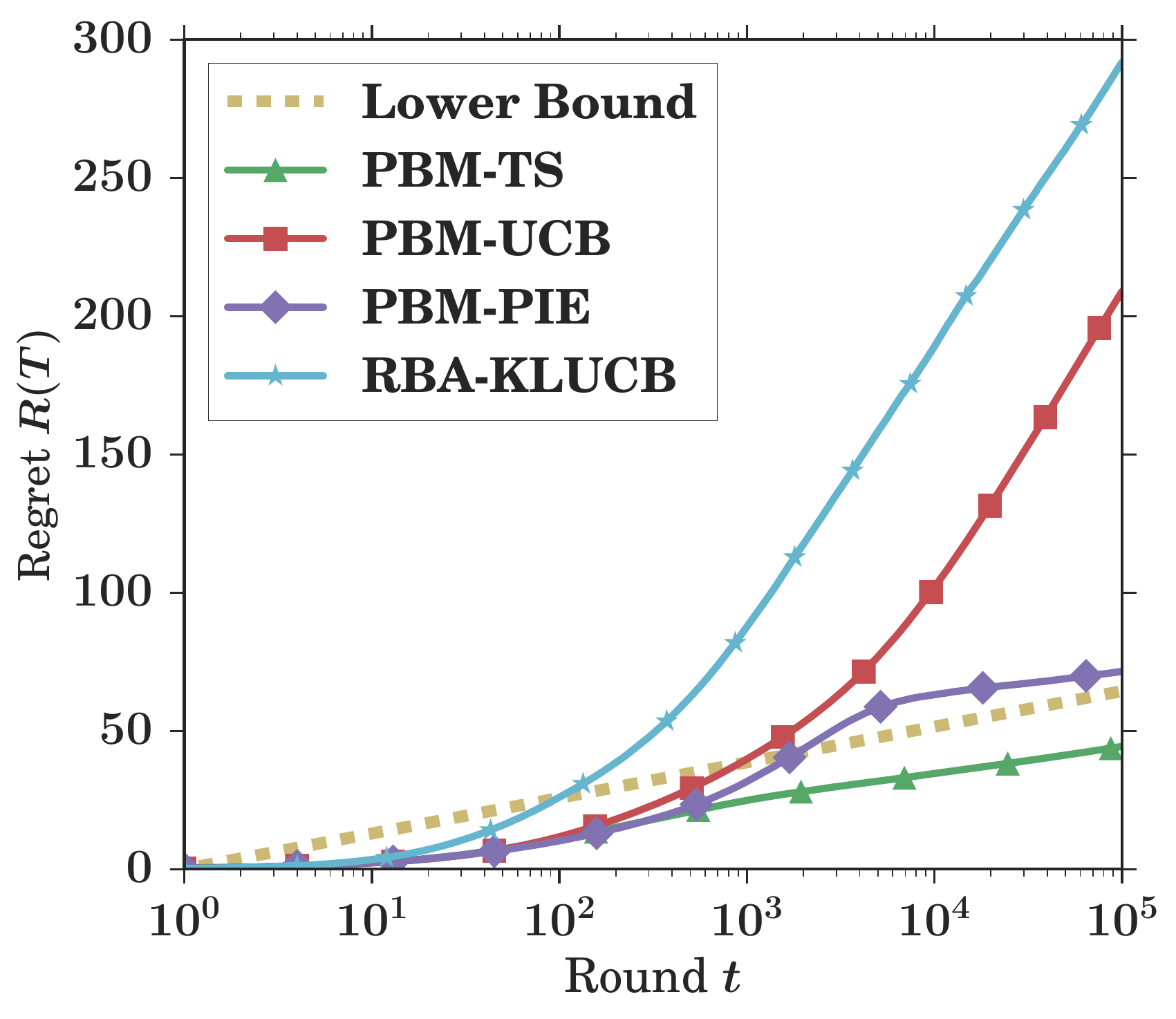}}
	\caption{Simulation results for the suggested strategies.}
\end{figure}

\subsection{Simulations}

In order to evaluate our strategies, a simple problem is considered in which
$K=5$, $L=3$, $\kappa=(0.9,0.6,0.3)$ and $\theta = (0.45,0.35,0.25,0.15,0.05)$.
The arm expectations are chosen such that the asymptotic behavior can be 
observed after reasonable time horizon. All results are averaged based on
$10,000$ independent runs of the algorithm. We present the results in
Figure~\ref{fig:DVA_simu} where \algonameCensUCB, 
\algonamePIE\ and \algonameTS\ are compared to \algonameRBA. The performance 
of \algonamePIE\ and \algonameTS\ are comparable, the latter even being under
the lower bound (it is a common observation, e.g. see
\cite{komiyama2015optimal}, and is due to the
asymptotic nature of the lower bound). The curves confirm our analysis
for \algonamePIE\ and lets us conjecture that the true Thompson Sampling
policy might be asymptotically optimal. As expected, 
\algonamePIE\ shows 
asymptotically optimal performance, matching the lower bound after a large 
enough horizon.

%
%
%
\begin{table}[t]
	\begin{minipage}[]{0.5\textwidth}
        \vspace{5mm}
		\begin{tabular}[b]{ c | c | c | c }
			{$\#$\textbf{ads} $\boldsymbol{(K)}$}  & 
			{$\boldsymbol{\#}$\textbf{records}} & {$\boldsymbol{\min~\theta}$} 
			& 
			{$\boldsymbol{\max~\theta}$}\\
			\hline
			\small $5$ &  \small $216,565$ &  \small $0.016$ &  \small  
			$0.077$\\ 
			\small $5$ &  \small $68,179$ &  \small $0.031$ &  \small $0.050$\\
			\small $6$ &  \small $435,951$ &  \small $0.025$ &  \small $0.067$\\
			\small $6$ &  \small $110,071$ &  \small $0.023$ &  \small $0.069$\\
			\small $6$ &  \small $147,214$ &  \small $0.004$ &  \small $0.148$\\
			\small $8$ &  \small $122,218$ &  \small $0.108$ &  \small $0.146$\\
			\small $11$ &  \small $1,228,004$ &  \small $0.022$ &  \small 
			$0.149$\\
			\small $11$ &  \small $391,951$ &  \small $0.022$ &  \small 
			$0.084$\\
			\hline
		\end{tabular}
		\vspace{1.75cm}
		\caption{\label{tab:kdd2} Statistics on the queries: each line 
			corresponds to the sub-dataset associated with a query.}
	\end{minipage}
	\quad
	\begin{minipage}[]{0.415\textwidth}	
		\includegraphics[width=\textwidth]{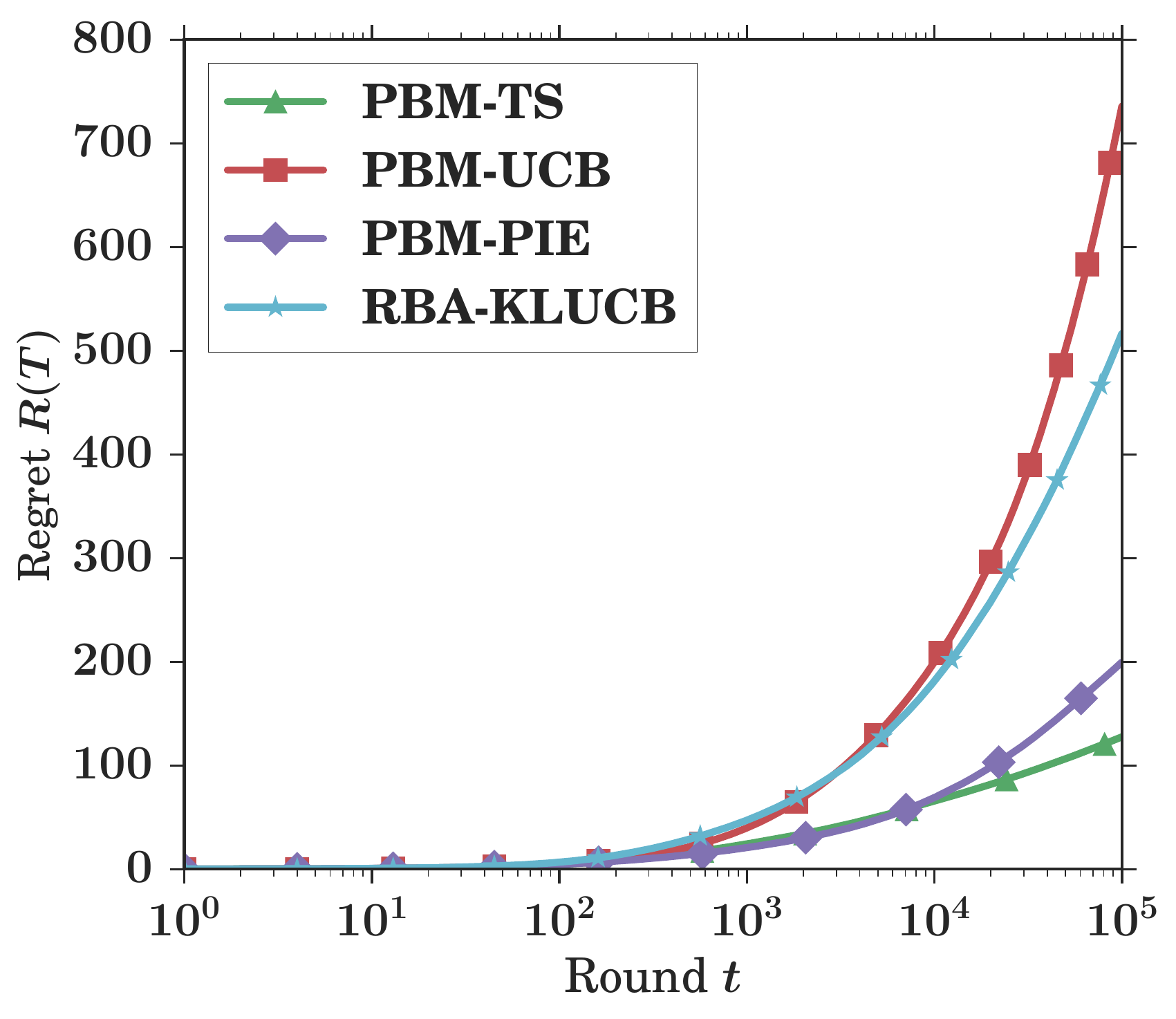}
		\captionof{figure}{\small Performance of the proposed algorithms under 
		the PBM on real data.}
		\label{fig:kdd2}
	\end{minipage}
\end{table}

\subsection{Real data experiments: search advertising}

The dataset  was provided for KDD Cup 2012 track 2
\footnote{http://www.kddcup2012.org/} and involves session logs of
soso.com, a search engine owned by Tencent. It consists of ads that were 
inserted among search results. Each of the $150M$ lines from the log
contains the user ID, the query typed, an ad, a position ($1$, $2$ or $3$) 
at which it was displayed and a binary reward (click/no-click). First,
for every query, we excluded ads that were not displayed at least $1,000$
times at every position. We also filtered queries that had less than $5$ ads
satisfying the previous constraints. As a result, we
obtained $8$ queries with at least $5$ and up to $11$ ads. For each query $q$, 
we computed the matrix of the average click-through rates (CTR): 
$M_q \in \mathds{R}^{K \times L}$, where $K$ is the number of ads for the query 
$q$ and $L=3$ the number of
positions. It is noticeable that the SVD of each $M_q$ matrix has a highly 
dominating first singular value, therefore validating the low-rank assumption 
underlying in the PBM. In order to 
estimate the parameters of the problem, we used the EM algorithm suggested by 
\cite{chuklin2015click,dempster1977maximum}. 
Table~\ref{tab:kdd2} reports some statistics about the bandit models
reconstructed for each query: number of arms $K$, amount of data used to 
compute the parameters, minimum and maximum values of the $\theta$'s for each 
model.

We conducted a series of $2,000$ simulations over this dataset. At the 
beginning of each run, a query was randomly selected together with
corresponding probabilities of scanning positions and arm expectations.
Even if rewards were still simulated, this scenario is more realistic
since the values of the parameters were extracted from a real-world dataset.
We show results for the different algorithms in Figure~\ref{fig:kdd2}. It is
remarkable that \algonameRBA\ performs slightly better than \algonameCensUCB.
One can imagine that \algonameCensUCB\ does not benefit enough from position
aggregations -- only $3$ positions are considered -- to beat \algonameRBA.
Both of them are outperformed by \algonameTS\ and \algonamePIE.

%


\section*{Conclusion}
\label{sec:conclusion}
%

This work provides the first complete analysis of the PBM in an online context. 
The proof scheme used
to obtain the lower bound on the regret is interesting on its
own, as it can be generalized to various other settings. The
tightness of the lower bound is validated by our analysis of \algonamePIE\ but
it would be an interesting future contribution to provide such guarantees for
more straightforward algorithms such as \algonameTS\ or a `PBM-KLUCB' using the
confidence regions of \algonamePIE.
In practice, the algorithms are robust to small variations of the
values of the $(\kappa_l)_l$, but it would be preferable to obtain
some control over the regret under uncertainty on these examination
parameters.


\newpage

\bibliographystyle{abbrv}
\bibliography{biblio}

\newpage
\appendix
%

\section{Properties of $\hat{\theta}_k(t)$ (Section~\ref{sec:model})}
\label{sec:appendix-estim}
Conditionnally to the actions $A(1)$ up to $A(t-1)$, the log-likelihood of the observations $Z(1), \dots, Z(t-1)$ may be written as
\begin{multline*}
\sum_{s=}^{t-1} \sum_{k=1}^K \sum_{l=1}^L \mathds{1}\{A_l(t)=k\} \left[ Z_l(t) \log (\kappa_l \theta_k) + (1-Z_l(t)) \log (1-\kappa_l \theta_k) \right] \\
= \sum_{k=1}^K \sum_{l=1}^L S_{k,l}(t) \log (\kappa_l \theta_k) + (N_{k,l}(t)-S_{k,l}(t)) \log (1-\kappa_l \theta_k).
\end{multline*}
Differenciating twice with respect to $\theta_k$ and taking the expectation of $(S_{k,l}(t))_{l}$, contional to $A(1), \dots, A(t-1)$, yields the expression of $I(\theta_k)$ given in Section~\ref{sec:model}.

\section{Proof of Theorem~\ref{th:LB_opt}}
\label{sec:proof-lower-bound}
\subsection{Proof of Lemma~\ref{lem:loglikelihood}}

\begin{proof}[\unskip\nopunct]
Under the PBM, the conditional expectation of the log-likelihood ratio defined
in~\eqref{eq:log-likelihood-ratio} writes
\begin{align*}
    \mathds{E}_\theta[\ell(t)|A(1),\ldots,A(t)] &= \mathds{E}_\theta\left[\sum_{s=1}^t
        \sum_{a \in \mathcal{A}} \mathds{1}\{A(s) = a\} \sum_{l=1}^L
        \log \frac{p_{a_l}(X_l(s)Y_l(s) ; \theta)}{p_{a_l}(X_l(s)Y_l(s);
        \lambda)}~\middle|~A(1),\ldots,A(t) \right] \\
        &= \sum_{s=1}^t \sum_{a \in \mathcal{A}} \mathds{1}\{A(s) = a\}
            \sum_{l=1}^L \mathds{E}\left[\log \frac{p_{a_l}(X_l(s)Y_l(s) ;
            \theta)}{p_{a_l}(X_l(s)Y_l(s); \lambda)}~\middle|~A(s) = a \right] \\
        &= \sum_{a \in \mathcal{A}} N_a(t) \sum_{l=1}^L\sum_{k=1}^{K}
          \mathds{1}\{a_l=k\}d(\kappa_l\theta_{k}, \kappa_l\lambda_{k}) \\
        &= \sum_{a \in \mathcal{A}} N_a(t) I_a(\theta,\lambda) ,
\end{align*}
using the notation
$I_a(\theta,\lambda) = \sum_{l=1}^L\sum_{k=1}^{K} 
\mathds{1}\{a_l=k\} d(\kappa_l\theta_{k}, \kappa_l\lambda_{k})$.
\end{proof}

\subsection{Details on the proof of Proposition~\ref{prop:constraintLB}}
\begin{lemma}\label{lem:boundLikelihood}
    Let $\theta=(\theta_1,\ldots,\theta_K)$ and 
    $\lambda=(\lambda_1,\ldots,\lambda_K)$ be two bandit models such that the 
    distributions of all
    arms in $\theta$ and $\lambda$ are mutually absolutely continuous.
    Let $\sigma$ be a
    stopping time with respect to $(\mathcal{F}_t)$ such that $(\sigma < +\infty)$
    a.s.~under both models. Let $\mathcal{E} \in \mathcal{F}_{\sigma}$ be an 
    event such that
    $0 < \mathds{P}_{\theta }(\mathcal{E}) < 1$. Then one has
    \begin{align*}
    	\sum_{a\in \mathcal{A}}  
    	I_a(\theta,\lambda)\mathds{E}_{\theta}[N_a(\sigma)]
     \geq d(\mathds{P}_{\theta}(\mathcal{E}), \mathds{P}_{\lambda}(\mathcal{E})),
    \end{align*}
    where $I_a(\theta,\lambda)$ is the conditional expectation of the
    log-likelihood ratio for the model of interest.
\end{lemma}
The proof of this lemma directly follows from the above expressions of the
log-likelihood ratio and from the proof of Lemma $1$ in Appendix A.1 of
\cite{kaufmann2015complexity}.

We simply recall the following technical lemma for completeness.
\begin{lemma}\label{lemma2}
    Let $\sigma$ be any stopping time with respect to $(\mathcal{F}_t)$.
    For every event $A \in \mathcal{F}_{\sigma}$,
    \begin{align*}
        \mathds{P}_{\lambda} (A)= \mathds{E}_{\theta}[\mathds{1}\{A\} \exp 
        (-\ell(\sigma))].
    \end{align*}
\end{lemma}
A full proof of Lemma~\ref{lemma2} can be found in the Appendix A.3 of  
\cite{kaufmann2015complexity} (proof of Lemma 15).

\subsection{Lower bound proof (Theorem~\ref{th:LB_opt})}

\begin{proof}[\unskip\nopunct]
In order to prove the simplified lower bound of Theorem \ref{th:LB_opt} we 
basically have two arguments:
\begin{enumerate}
	\item a lower bound on $f(\theta)$ can be obtained by enlarging the 
	    feasible set, that is by relaxing some constraints;
	\item Lemma \ref{lem:c_prime_lp} can be used to lower bound the 
	    objective function of the problem.
\end{enumerate}
The constant $f(\theta)$ is defined by 
\begin{gather}
f(\theta) =  \inf_{c\succeq 0}  \sum_{a\neq a^*(\theta)} 
\Delta_{a}(\theta) c_a \label{opt_arg} \\ 
s.t ~~ \inf_{\lambda \in B(\theta)} \sum_{a\in 
    \mathcal{A}} I_a(\theta,\lambda)c_a \geq 1 .\label{opt_cons}
\end{gather} 
We begin by relaxing some constraints: we only allow the change of measure 
$\lambda$ to belong to the sets $B_k(\theta):= \left\lbrace \lambda \in \Theta 
| 
\forall j\neq k ,\theta_j=\lambda_j \text{ and } \mu^*(\theta)<\mu^*(\lambda) 
\right\rbrace$ defined in Section 
\ref{sec:lower_bound}:
\begin{gather}
f(\theta) =  \inf_{c\succeq 0}  \sum_{a\neq a^*(\theta)} 
\Delta_{a}(\theta) c_a  \\ 
s.t ~~ \forall k\notin a^*(\theta),~\forall \lambda \in B_k(\theta), 
\sum_{a\in 
	\mathcal{A}} I_a(\theta,\lambda)c_a \geq 1 .\label{opt_cons2}
\end{gather}
The $K-L$ constraints (\ref{opt_cons2}) only let one parameter move and must be 
true for any value satisfying the definition of the corresponding set 
$B_k(\theta)$. In practice, for each $k$, the parameter $\lambda_k$ must be set 
to at least $\theta_L$. Consequently, these constraints may then be rewritten 

\begin{gather}
    f(\theta) = \inf_{c\succeq 0} \sum_{a\neq a^*(\theta)} \Delta_a(\theta)
        c_a \label{ineq:apm_objective} \\
    s.t~ \forall k\notin a^*(\theta), \sum_{a\neq a^*(\theta)}c_a\sum_{l=1}^L
        \mathds{1}\{a_l=k\} d(\kappa_l\theta_k,\kappa_l\theta_L)\geq 1.
        \label{ineq:apm_constraints}
\end{gather}
Proposition~\ref{prop:ca} tells
us that coefficients $c_a$ are all zeros except for actions $a\in\mathcal{A}$
which can be written $a = v_{k,l_k}$ where $l_k = \argmin_{l \leq L}
\frac{\Delta_{v_{k,l}}(\theta)}{d(\kappa_l\theta_k,\kappa_l\theta_L)}$.
Thus, we obtain the desired lower bound by rewriting (\ref{ineq:apm_objective}) as
$$
    f(\theta) \geq \sum_{k=L+1}^K \min_{l \in \{1,\ldots,L\}} 
    \frac{\Delta_{v_{k,l}}(\theta)}{d(\kappa_l\theta_k,\kappa_l\theta_L)}.
$$
\end{proof}

\begin{proposition} \label{prop:ca}
    Let $c = \{c_a~:~a\neq a^*\}$ be a solution of the linear problem (LP) in
    Theorem~\ref{th:LB_opt}. Coefficients are all zeros except for actions $a$
    which can be written as $a=(1,\ldots,l_k-1,k,l_k,\ldots,L-1):=v_{k,l_k}$
    where $k > L$ and $l_k = \argmin_{l \leq L} \frac{\Delta_{v_{k,l}}(\theta)}
    {d(\kappa_l\theta_k,\kappa_l\theta_L)}$.
\end{proposition}

\begin{proof}
    We denote by $\pi_k(a)$ the position of item $k \in \{1,\ldots,K\}$ in action
    $a$ ($0$ if $k \notin a$). Let $l_k$ be the optimal position of item
    $k > L$ for exploration: $l_k = \argmin_{l\leq L} \frac{\Delta_{v_{k,l}}(\theta)}
    {d(\kappa_l\theta_k,\kappa_l\theta_L)}$.
    Following \cite{combes2015learning}, we show by contradiction that $c_a>0$
    implies that $a$ can be written $v_{k,l_k}$ for a well chosen $k > L$. Let $\alpha
    \neq a^*$ be a suboptimal action such that $\forall k > L, \alpha \neq v_{k,l_k}$
    and $c_{\alpha} > 0$. We need to show a contradiction.  Let us introduce a new set
    of coefficients $c'$ defined as follows, for any $a \neq a^*$:
    $$
    c'_a =
      \begin{cases}
        0                                   & \quad \text{if } a=\alpha         \\
        c_a + 
        \frac{d(\kappa_{\pi_k(\alpha)}\theta_k,\kappa_{\pi_k(\alpha)}\theta_L)}
                {d(\kappa_{l_k}\theta_k,\kappa_{l_k}\theta_L)}c_\alpha  & \quad 
                \text{if }
                \exists k > L \text{ s.t. } a=v_{k,l_k} \text{ and } k \in \alpha \\
        c_a \                               & \quad \text{otherwise.}           \\
      \end{cases}
    $$
    According to Lemma \ref{lem:c_prime_lp}, these coefficients satisfy the constraints
    of the LP. We now show that these new coefficients yield a strictly lower value to
    the optimization problem:
    \begin{align}
    c(\theta) - c'(\theta) &= c_\alpha \Delta_\alpha(\theta) - \sum_{k > L:k \in \alpha}
            \frac{d(\kappa_{\pi_k(\alpha)}\theta_k,\kappa_{\pi_k(\alpha)}\theta_L)}
            {d(\kappa_{l_k}\theta_k,\kappa_{l_k}\theta_L)}c_\alpha 
            \Delta_{v_{k,l_k}}(\theta) \nonumber \\
        & > c_\alpha \left( \sum_{k>L: k \in \alpha}
            \Delta_{v_{k,\pi_k(\alpha)}}(\theta) -
            \sum_{k > L:k \in \alpha} \frac{d(\kappa_{\pi_k(\alpha)}\theta_k,
            \kappa_{\pi_k(\alpha)}\theta_L)} {d(\kappa_{l_k}\theta_k,
            \kappa_{l_k}\theta_L)}\Delta_{v_{k,l_k}}(\theta)\right). \label{ineq:apm_1}
    \end{align}
    The strict inequality (\ref{ineq:apm_1}) is shown in Lemma \ref{lem:ineqLB1}.
    Let $k > L$ be one of the suboptimal arms in $\alpha$. By definition of $l_k$,
    the corresponding term of the sum in equation (\ref{ineq:apm_1}) is positive.
    Thus, we have that $c(\theta) > c'(\theta)$ and, hence, by contradiction, we showed that 
    $c_a > 0$ iff $a$ can be written $a=v_{k,l_k}$ for some $k>L$.
\end{proof}

\begin{lemma} \label{lem:c_prime_lp}
    Let $c$ be a vector of coefficients that satisfy constraints  
    (\ref{ineq:apm_constraints}) of the optimization problem. Then, coefficients $c'$ as 
    defined in Proposition \ref{prop:ca} also satisfy the constraints:
    $$
    \forall k\notin a^*(\theta), \sum_{a\neq a^*(\theta)}c'_a\sum_{l=1}^{L}
              \mathds{1}\{a_l=k\} d(\kappa_l\theta_k,\kappa_l\theta_L)\geq 1.
    $$
\end{lemma}

\begin{proof}
    We use the same $\alpha$ as introduced in Proposition \ref{prop:ca}. Let us fix
    $k \notin a^*(\theta)$. Let us define
    $$L(c) = \sum_{a\neq a^*(\theta)}c_a\sum_{l=1}^{L} \mathds{1}\{a_l=k\}
             d(\kappa_l\theta_k,\kappa_l\theta_L).$$
    We have
    \begin{align*}
        L(c')-L(c) = -c_\alpha \sum_{l=1}^L \mathds{1}\{\alpha_l=k\} d(\kappa_l\theta_k,
                \kappa_l\theta_L) + \sum_{l:\alpha_l>L}
                &\frac{d(\kappa_l\theta_k,\kappa_l\theta_L)}
                {d(\kappa_{l_k}\theta_k,\kappa_{l_k}\theta_L)}c_\alpha \\
            &\times \mathds{1}\{\alpha_l=k\}d(\kappa_{l_k}\theta_k,\kappa_{l_k}\theta_L).
    \end{align*}
    If $k \notin \alpha$, clearly, $L(c')-L(c)=0$. Else, $k \in \alpha$ and we note $p$
    its position in $\alpha$: $p = \pi_k(\alpha)$. We rewrite:
    \begin{equation*}
        L(c')-L(c) = c_\alpha d(\kappa_p\theta_k,\kappa_p\theta_L)\left( -1 +
                     \frac{d(\kappa_{l_k}\theta_k,\kappa_{l_k}\theta_L)}
                     {d(\kappa_{l_k}\theta_k,\kappa_{l_k}\theta_L)} \right)
                     = 0.
    \end{equation*}
    Thus, the coefficients $c'$ satisfy the constraints from Proposition 
    \ref{prop:ca}.
\end{proof}

\begin{lemma} \label{lem:ineqLB1}
    Let $\alpha$ be as in the proof of Proposition \ref{prop:ca}.
    $$
    \Delta_\alpha(\theta) > \sum_{k>L: k \in \alpha} \Delta_{v_{k,\pi_k(\alpha)}}(\theta).
    $$
\end{lemma}

\begin{proof}
    Let $k_1,\ldots,k_p$ be the suboptimal arms in $\alpha$ by increasing 
    position. Let $v(\alpha)$ be the action in $\mathcal{A}$ with lower regret 
    such that it
    contains all the suboptimal arms of $\alpha$ in the same positions. Thus,
    $v(\alpha) = (1,\ldots,\pi_{k_1}(\alpha)-1,k_1,\pi_{k_1}(\alpha),\ldots,
    \pi_{k_2}(\alpha)-2,k_2,\pi_{k_2}(\alpha)-1,\ldots,L-p)$. By definition,
    one has that $\Delta_\alpha(\theta)
    \geq \Delta_{v(\alpha)}(\theta)$. In the following, we show that
    $\Delta_{v(\alpha)}(\theta) \geq \sum_{k>L: k \in \alpha} \Delta_{v_{k,\pi_k(\alpha)}}
    (\theta)$ for $p=2$ (that is to say $\alpha$ contains $2$ suboptimal arms
    $k_1$ and $k_2$).

    For the sake of readability, we write $\pi_i$ instead of $\pi_{k_i}(\alpha)$
    in the following.
    \begin{align*}
        \Delta_{v(\alpha)}(\theta) &= \sum_{l=1}^L \kappa_l
            (\theta_l-\theta_{(v_{k_1,\pi_1})_l}) + \sum_{l=1}^L
            \kappa_l(\theta_{(v_{k_1,\pi_1})_l} - \theta_{v(\alpha)_l}) \\
        &= \Delta_{v_{k_1,\pi_1}}(\theta) + \left[\kappa_{\pi_2}\theta_{\pi_2-1}
            + \ldots + \kappa_L\theta_{L-1}\right] - \left[\kappa_{\pi_2}
            \theta_{k_2} + \kappa_{\pi_2+1} \theta_{\pi_2-1} + \ldots
            + \kappa_L\theta_{L-2} \right] \\
        &= \Delta_{v_{k_1,\pi_1}}(\theta) + \Delta_{v_{k_2,\pi_2}}(\theta) + 
            \left[\kappa_{\pi_2}(\theta_{\pi_2-1} - \theta_{\pi_2}) + \ldots + 
            \kappa_L(\theta_{L-1}-\theta_L) \right] - \\
        &~~~~\left[\kappa_{\pi_2+1}(\theta_{\pi_2-1} - \theta_{\pi_2})
            + \ldots + \kappa_L (\theta_{L-2} - \theta_{L-1}) \right] \\
        &= \Delta_{v_{k_1,\pi_1}}(\theta) + \Delta_{v_{k_2,\pi_2}}(\theta)
            + \mathcal{R}(\theta).
    \end{align*}
    Thus, one has to show that $\mathcal{R}(\theta) = \kappa_{\pi_2}
    (\theta_{\pi_2-1}-\theta_{\pi_2}) + \kappa_{\pi_2+1}(2\theta_{\pi_2}
    -\theta_{\pi_2-1}-\theta_{\pi_2+1}) + \ldots + \kappa_L(2\theta_{L-1}-
    \theta_{L-2}-\theta_L) > 0$. In fact, using that
    $\kappa_l \geq \kappa_{l+1}$ for all $l < L$, we have
    \begin{align*}
        \mathcal{R}(\theta) &\geq \kappa_{\pi_2+1}(\theta_{\pi_2-1}-\theta_{\pi_2}
            + 2\theta_{\pi_2}-\theta_{\pi_2-1}-\theta_{\pi_2+1}) + \ldots + 
            \kappa_L(2\theta_{L-1} - \theta_{L-2}-\theta_L) \\
        &\geq \kappa_{\pi_2+2}(\theta_{\pi_2+1}-\theta_{\pi_2+2}) + \ldots +
            \kappa_L(2\theta_{L-1} - \theta_{L-2}-\theta_L) \\
        &\geq \ldots \\
        &\geq \kappa_L (\theta_{L-1} - \theta_L) \\
        &> 0.
    \end{align*}
\end{proof}

\section{Proof of Proposition~\ref{prop:boundCensUCB}}
\label{sec:concentrationUCB}

In this section, we fix an arm $k\in \{1,\dots,K\}$ and obtain an upper confidence bound for the estimator $\hat{\theta}_k(t):=S_k(t)/\tilde{N}_k(t)$. 
Let $\tau_i$ be the instant of the $i$-th draw of arm 
$k$ (the $\tau_i$ are stopping times w.r.t. $\mathcal{F}_t$). We introduce the 
centered sequence of successive observations from arm $k$
\begin{equation}
  \bar{Z}_{k,i} = \sum_{l=1}^L \mathds{1}\{ A_{l}(\tau_i)=k\} (X_{l}(\tau_i) 
Y_{l}(\tau_i) - \theta_k \kappa_l).
  \label{eq:Zbar}
\end{equation}
Introducing the filtration $\mathcal{G}_i = 
\mathcal{F}_{\tau_{i+1} -1}$, one has $\mathds{E}[\bar{Z}_{k,i} | \mathcal{G}_{i -1}] = 0$, and therefore, the 
sequence 
\[M_{k,n} = \sum_{i=1}^{n} \bar{Z}_{k,i}\]
is a martingale with bounded increments, w.r.t. the filtration 
$(\mathcal{G}_n)_n$. By 
construction, one has
\[M_{k,N_k(t)} = S_k(t) - \tilde{N}_k(t) \theta_k = \tilde{N}_k(t) (\hat{\theta}_k(t) -\theta_k). \]

We use the so-called peeling technique together with the maximal version of 
Azuma-Hoeffding's inequality \cite{boucheron2013concentration}. For any $\gamma 
> 0$ one has
\begin{align*}
\mathds{P}\left( M_{k,N_k(t)} < -\sqrt{N_k(t) \delta/2}\right) & \leq  
\sum_{i=1}^{\frac{\log(t)}{\log(1+\gamma)}} \mathds{P}\left(M_{k,N_k(t)} < 
    -\sqrt{N_k(t) \delta/2} \, , N_k(t) \in [(1+\gamma)^{i-1},(1+\gamma)^i)\right) 
\\
& \leq  \sum_{i=1}^{\frac{\log(t)}{\log(1+\gamma)}} \mathds{P}\left(\exists 
i \in \{1,\dots,(1+\gamma)^i \} : \ M_{k,i} < -\sqrt{(1+ \gamma)^{i-1} 
	\delta/2} \right) \\
& \leq \sum_{i=1}^{\frac{\log(t)}{\log(1+\gamma)}} 
\exp\left(-\frac{\delta(1+\gamma)^{i-1}}{(1+\gamma)^i}\right) = 
\frac{\log(t)}{\log(1+\gamma)}\exp\left(-\frac{\delta 
}{(1+\gamma)}\right).
\end{align*}

Choosing $\gamma = 1/(\delta-1)$, gives 
\[\mathds{P}\left( \hat{\theta}_k(t) - \theta_k
< -\frac{\sqrt{N_k(t)\delta/2}}{\tilde{N}_k(t)}\right) \leq 
\delta e \log(t) e^{-\delta}.\]

\section{Regret analysis for \algonameCensUCB\ (Theorem~\ref{th:pbm-ucb})}
\label{sec:proofUCB}

We proceed as Kveton et al. (2015) \cite{kveton2015tight}. We start by 
considering separately 
rounds when one of the confidence intervals is violated. We denote by $B_{t,k}=\sqrt{N_k(t)(1+\epsilon)\log t/2}/\tilde{N}_k(t)$ the 
PBM-UCB exploration bonus and by $B_{t,k}^+=\sqrt{N_k(t)(1+\epsilon)\log T/2}/\tilde{N}_k(t)$ an upper bound of this bonus (for $t\leq T$). We define the event 
$E_t = \{\exists k\in A(t) \,:\, \vert 
\hat{\theta}_k(t)-\theta_k\vert > B_{t,k}\}$. Then, the regret can be 
decomposed into 
$$ R(T) = 
\sum_{t=1}^{T}\Delta_{A(t)}\mathds{1}_{E_t} + 
\Delta_{A(t)}\mathds{1}_{\bar{E_t}} .$$
and, similarly to \cite{kveton2015tight} (Appendix A.1), the first 
term of this sum can be bounded from above in expectation by a constant $C_0(\epsilon)$ that does 
not depend 
on $T$ using Proposition~\ref{prop:boundCensUCB}. So, it remains to bound 
the regret suffered even when confidence 
intervals are respected, that is the sum on the r.h.s of
$$ \mathbb{E}[R(T)]< C_0(\epsilon) + 
\mathbb{E}[\sum_{t=1}^{T}\Delta_{A(t)}\mathds{1}\{\bar{E_t},\Delta_{A(t)}>0\}] 
.$$

It can be done using techniques from 
\cite{combes2015combinatorial,kveton2015tight}. We start by defining events 
$F_t$, $G_t$, $H_t$ in order to decompose the 
part of the regret at stake. Then, we show an equivalent of Lemma 2 of 
\cite{kveton2015tight} for our 
case and finally we refer to the proof of Theorem 3 in Appendix A.3 of 
\cite{kveton2015tight}. 

For each round $t\geq 1$, we define the set of arms $S_t = \{1\leq l\leq L : 
N_{A_l(t)}(t) \leq \frac{8(1+\epsilon)\log T 
	\left(\sum_{s=1}^{L} 
	\kappa_s\right)^2}{\kappa_L^2 \Delta_{A(t)}^2}\}$ and the related events 
\begin{itemize}
	\item $F_t = \{\Delta_{A(t)}>0,\, \Delta_{A(t)}\leq 2\sum_{l=1}^{L} 
	\kappa_{l}B_{t,A_l(t)}^+\}$;
	\item $G_t = \{\vert S_t \vert \geq l \}$;
	\item $H_t = \{\vert S_t \vert < l \, , \, \exists k\in A(t), 
	N_k(t)\leq \frac{8(1+\epsilon)\log T\left(\sum_{s=1}^{l} 
		\kappa_s\right)^2}{\kappa_L^2\Delta_{A(t)}^2} \} $, where 
	the constraint on $N_k(t)$ only differs from the first one by its 
	numerator which is smaller than the previous one, leading to an 
	even stronger constraint. 		
\end{itemize} 

\begin{fact}
	According to Lemma 1 in \cite{kveton2015tight}, the 
	following inequality is still valid with our own definition of $F_t$ :
    $$ \sum_{t=1}^{T} \Delta_{A(t)}\mathds{1}\{\bar{E_t},\Delta_{A(t)}>0\} 
    \leq  
	\sum_{t=1}^{T} \Delta_{A(t)}\mathds{1}\{F_t \}.
	$$
\end{fact}

\begin{proof}
	Invoking Lemma 1 from \cite{kveton2015tight} needs to be justified as 
    our setting is quite different. Taking action $A(t)$ means that 
	$$ \sum_{l=1}^{L} \kappa_{l}U_{A_l(t)}(t) \geq \sum_{l=1}^{L}\kappa_l 
	U_l(t).$$
	Under event $\bar{E}_t$, all UCB's are above the true parameter 
	$\theta_k$ so we have 
	$$\sum_{l=1}^{L} \kappa_{l}(\theta_{A_l(t)} + 2B_{t,A_l(t)}) 
	\geq 
	\sum_{l=1}^{L}\kappa_l (\theta_l + B_{t,l})\geq 
	\sum_{l=1}^{L}\kappa_l \theta_l.$$
	Rearranging the terms above and using $B_{t,l(t)}\leq B_{t,l(t)}^+$, 
	we obtain
	$$\sum_{l=1}^{L}\kappa_{l}B_{t,A_l(t)}^+ \geq  
	2\sum_{l=1}^{L}\kappa_{l}B_{t,A_l(t)} \geq \Delta_{A(t)}.$$
\end{proof}

We now have to prove an equivalent of Lemma 2 in 
\cite{combes2015combinatorial} that would allow us to split 
the right-hand side above in two parts. Let us show that $F_t \subset 
(G_t\cup H_t)$  by showing its contrapositive: if $F_t$ is true then we 
cannot have $(\bar{G_t}\cap \bar{H_t})$. Assume both of these events are 
true. 
Then, we have 
\begin{align*}
\Delta_{A(t)} &\overset{F_t}{\leq} 2\sum_{l=1}^{L} 
\kappa_{l}B_{t,A_l(t)}^+\\
& \leq 2\sum_{l=1}^{L}\kappa_{l} 
\sqrt{\frac{N_{A_l(t)}(t)}{\tilde{N}_{A_l(t)}(t)}}
\sqrt{\frac{(1+\epsilon)\log(T)}{2\tilde{N}_{A_l(t)}(t)}}\\
& =2 \sum_{l=1}^{L}\kappa_{l} 
\frac{N_{A_l(t)}(t)}{\tilde{N}_{A_l(t)}(t)}
\sqrt{\frac{(1+\epsilon)\log(T)}{2N_{A_l(t)}(t)}}\\
& \leq \frac{\sqrt{2(1+\epsilon)\log T}}{\kappa_L}\sum_{l=1}^{L} 
\frac{\kappa_{l}}{\sqrt{N_{A_l(t)}(t)}}\\
& = \frac{\sqrt{2(1+\epsilon)\log T}}{\kappa_L}\left(\sum_{l\notin 
	 S_t}\frac{\kappa_{l}}{\sqrt{N_{A_l(t)}(t)}}
+ \sum_{l\in S_t}\frac{\kappa_{l}}{\sqrt{N_{A_l(t)}(t)}}\right)\\
& \overset{(\bar{G_t}\cap \bar{H_t})}{<} \frac{\sqrt{2(1+\epsilon)\log 
		T}}{\kappa_L}\frac{\kappa_L\Delta_{A(t)}}{  
	2\sqrt{2(1+\epsilon)\log T}}
\left(\frac{\sum_{l\notin S_t} 
	\kappa_{l}}{\sum_{s=1}^{L}\kappa_s} +\frac{\sum_{l\in 
		S_t}\kappa_{l}}{\sum_{s=1}^{l}\kappa_s}\right)\\
& \leq \Delta_{A(t)}
\end{align*}
which is a contradiction. 
The end of the proof proceeds exactly as in the end of the proof of Theorem 
6 in of \cite{combes2015combinatorial}: events $G_t$ and $H_t$ are split 
into subevents corresponding to rounds where each specific suboptimal arm 
of the list is in $S_t$ or verifies the condition of $H_t$. We define 
\begin{align*}
G_{k,t} & =  G_t \cap \{k\in A(t),\, N_k(t)\leq \frac{8(1+\epsilon)\log T 
	\left(\sum_{s=1}^{L} 
	\kappa_s\right)^2}{\kappa_L^2 \Delta_{A(t)}^2}\ \},\\
H_{k,t} & = H_t \cap \{k\in A(t),\, N_k(t) \leq \frac{8(1+\epsilon)\log 
	T\left(\sum_{s=1}^{l} 
	\kappa_s\right)^2}{\kappa_L^2\Delta_{A(t)}^2} \}.
\end{align*}
The way we defined these subevents allows to write the two following bounds 
:
$$\sum_{k=1}^{K} \mathds{1}\{G_{k,t}\} = 
\mathds{1}\{G_t\}\sum_{k=1}^{K}\mathds{1}\{k\in S_t\}\geq 
l\mathds{1}\{G_t\}$$
so $\mathds{1}\{G_t\}\leq \sum_k \mathds{1}\{G_{k,t}\}/l$. And,
$$\mathds{1}\{H_t\}\leq \sum_{k=1}^{K}\mathds{1}\{H_{k,t}\}.$$
We can now bound the regret using these two results:
\begin{align*}
\sum_{t=1}^{T}\Delta_{A(t)}(\mathds{1}\{G_t\}+\mathds{1}\{H_t\}) & \leq 
\sum_{t=1}^T \sum_{k=1}^{K}\frac{\Delta_{A(t)}}{l}\mathds{1}\{G_{k,t}\} 
+\sum_{t=1}^{T}\sum_{k=1}^K \Delta_{A(t)} \mathds{1}\{H_{k,t}\}\\
&=  \sum_{t=1}^T 
\sum_{k=1}^{K}\frac{\Delta_{A(t)}}{l}\mathds{1}\{G_{k,t},A(t)\neq a^*\} 
+\sum_{t=1}^{T}\sum_{k=1}^K \Delta_{A(t)}\mathds{1}\{H_{k,t},A(t)\neq 
a^*\}.
\end{align*}
For each arm $k$, there is a finite number $C_k:=|\mathcal{A}_k|$ of 
actions in 
$\mathcal{A}$ containing $k$; we order them such that the corresponding 
gaps are in decreasing order $\Delta_{k,1}\geq \ldots \geq \Delta_{k,C_k}>0$. So 
we decompose each sum above on the different actions $A(t)$ possible:
\begin{align*}
\ldots & \leq \sum_{t=1}^T 
\sum_{k=1}^{K} 
\sum_{a\in 
	\mathcal{A}_k}\frac{\Delta_{k,a}}{l}\mathds{1}\{G_{k,t},A(t)=a\} 
+\sum_{t=1}^{T}\sum_{k=1}^K \sum_{a\in \mathcal{A}_k}\Delta_{k,a} 
\mathds{1}\{H_{k,t},A(t)=a\}. \\		
\end{align*}
The two sums on the right hand side look alike. For arm $k$ fixed, events 
$G_{k,t}$ and $H_{k,t}$ imply almost the same condition on $N_k(t)$, only 
$H_{k,t}$ is stronger because the bounding term is smaller. We now rely on 
a 
technical result by \cite{combes2015combinatorial} that allows to bound 
each sum. 
\begin{lemma}(\cite{combes2015combinatorial}, Lemma 2 in Appendix B.4)
	Let $k$ be a fixed item and $|\mathcal{A}_k|\geq 1$, $C>0$, we have	
	
	\[ \sum_{t=1}^{T}\sum_{a \in \mathcal{A}_k} 
	\mathds{1}\{k\in A(t),\, N_k(t)	\leq C/\Delta_{k,a}^2,\,A(t)=a\} 
	\Delta_{k,a} 
	\leq \frac{2C}{\Delta_{\text{min},k}}\]	
	
	where $\Delta_{\text{min},k}$ is the smallest gap among all suboptimal 
	actions containing arm $k$. In particular, when $k\notin a^*$ the 
	smallest gap is $\Delta_{\text{min},k} = \kappa_L(\theta_L-\theta_k)$. 
	While, when $k\in a^*$ it is less obvious what the minimal gap is, 
	however it corresponds the second best action $A_2$ containing only 
	optimal arms: $\Delta_{\text{min},k} = \Delta_{A_2}$.		
\end{lemma}

%
%
%
%
%

So, bounding each sum with the above lemma, we obtain
\[ \sum_{t=1}^{T}\Delta_{A(t)}(\mathds{1}\{G_t\}+\mathds{1}\{H_t\})
\leq \frac{16(1+\epsilon)\log T }{\kappa_L^2 } 
\underbrace{\left( 
	\frac{\left(\sum_{s=1}^{L} 
		\kappa_s\right)^2}{l} + \left(\sum_{s=1}^{l} 
	\kappa_s\right)^2 \right)}_{C(l;\kappa)}
\left( \frac{L}{\Delta_{A_2}} +\sum_{k\notin 
	a^*}\frac{1}{\kappa_L(\theta_L-\theta_k)}\right). \]
This bound can be optimized by minimizing $C(l;\kappa)$ over $l$.


\section{Regret analysis for \algonamePIE\ (Theorem~\ref{th:pbm-pie})}
\label{sec:proofUBPIE}

The proof follows the decomposition of \cite{combes2015learning}. For all
$t \geq 1$, we denote $f(t, \epsilon) = (1 + \epsilon)\log t$.

\subsection{Controlling leaders and estimations}

Define $\eta_0 = \min_{k \in \{1,\ldots,L-1\}} (\theta_k - \theta_{k+1}) / 2$ and let
$\eta < \eta_0$. We define the following set of rounds
$$A = \{t \geq 1 : \mathcal{L}(t) \neq (1,\ldots,L)\} .$$

Our goal is to upper bound the expected size of $A$. Let us introduce the following
sets of rounds:
\begin{align*}
    B &= \{t \geq 1 : \exists k \in \mathcal{L}(t), |\hat{\theta}_k(t) - \theta_k|
    \geq \eta \} ,\\
    C &= \{t \geq 1 : \exists k \leq L , U_k(t) \leq \theta_k \} ,\\
    D &= \{t \geq 1 : t \in A\setminus (B \cup C), \exists k \leq L, k \notin
    \mathcal{L}(t), |\hat{\theta}_k(t) - \theta_k| \geq \eta\} .
\end{align*}

We first show that $A \subset (B\cup C \cup D)$. Let $t \in A \setminus (B 
\cup C)$.
Let $k, k' \in \mathcal{L}(t)$ such that $k < k'$. Since $t \notin B$, we have that
$|\hat{\theta}_k(t) - \theta_k| \leq \eta$ and
$|\hat{\theta}_{k'}(t) - \theta_{k'}| \leq \eta$. Since $\eta \leq (\theta_k -
\theta_{k'}) / 2$, we conclude that $\hat{\theta}_k(t) \geq \hat{\theta}_{k'}(t)$.
This proves that $(\mathcal{L}_1(t),\ldots,\mathcal{L}_L(t)$ is an increasing
sequence. We have that
$\mathcal{L}_L(t) > L$ otherwise $\mathcal{L}(t) = (1,\ldots,L)$ which is a
contradiction because $t \in A$. Since $\mathcal{L}_L(t) > L$, there exists
$k \leq L$ such that $k \notin \mathcal{L}(t)$. We show by contradiction that
$|\hat{\theta}_k(t) - \theta_k| \geq \eta$. Assume that 
$|\hat{\theta}_k(t) - \theta_k| \leq \eta$. We also have that
$\hat{\theta}_{\mathcal{L}_L(t)}(t) - \theta_{\mathcal{L}_L(t)} \leq \eta$
because $\mathcal{L}_L(t) \in \mathcal{L}(t)$ and $t \notin B$.
Thus, $\hat{\theta}_k(t) > \hat{\theta}_{\mathcal{L}_L(t)}(t)$. We
have a contradiction because this would imply that $k \in \mathcal{L}(t)$. 
Finally we have proven that if $t\in A \setminus (B\cup C)$, then $t\in D$ 
so $A \subset (B\cup C \cup D)$.

By a union bound, we obtain
$$ \mathds{E}[|A|] \leq \mathds[|B|] + \mathds[|C|] +\mathds[|D|] . $$
In the following, we upper bound each set of rounds individually.

\paragraph*{Controlling $\mathds{E}[|B|]$:}
We decompose $B = \bigcup_{k=1}^K (B_{k,1} \cup B_{k,2})$ where
\begin{align*}
    B_{k,1} = \{t \geq 1 : k \in \mathcal{L}(t), \mathcal{L}_L(t) \neq k,
        |\hat{\theta}_k(t) - \theta_k| \geq \eta\} \\
    B_{k,2} = \{t \geq 1 : k \in \mathcal{L}(t), \mathcal{L}_L(t) = k,
        |\hat{\theta}_k(t) - \theta_k| \geq \eta\}
\end{align*}

\underline{Let $t \in B_{k,1}$:} $k \in A(t)$ so $\mathds{E}[k\in A(t)|t\in
B_{k,1}] = 1$. Furthermore, for all $t$, $\mathds{1}\{t \in B_{k,1}\}$ is
$\mathcal{F}_{t-1}$ measurable. Then we can apply Lemma \ref{lem:pie} (with
$H = B_{k,1}$ and $c = 1$).
$$\mathds{E}[|B_{k,1}|] \leq 2 (2 + \kappa_L^{-2}\eta^{-2}) .$$ 

\underline{Let $t\in B_{k,2}$:} $k\in \mathcal{B}(t)$ but because of the 
randomization of the algorithm, $k \in A(t)$ with probability $1/2$, i.e.
$\mathds{E}[k\in A(t)|t\in B_{k,2}] \geq 1/2$. We get
$$\mathds{E}[|B_{k,2}|] \leq 4(4 + \kappa_L^{-2}\eta^{-2})$$ 

By union bound over $k$, we get $\mathds{E}[|B|] \leq 
2K(10+3\kappa_L^{-2}\eta^{-2})$.

\paragraph*{Controlling $\mathds{E}[|C|]$:}
We decompose $C = \bigcup_{k=1}^L C_k$ where $C_k = \{t\geq 1 : U_k(t) \leq 
\theta_k\}$

We first require to prove Proposition \ref{prop:PIEcontrol}.


\begin{proof}
	Theorem 2 of \cite{magureanu2014lipschitz} implies that 
	\[\mathbb{P}\left( \sum_{l=1}^{L} 
	N_{k,l}(t)d(\frac{S_{k,l}(t)}{N_{k,l}(t)},\kappa_l\theta_k)\geq \delta 
	\right)\leq 
	e^{-\delta}\left(\frac{\left\lceil\delta\log(t)\right\rceil\delta}{L} 
	\right)^Le^{L+1}. 
	\]
	The function $\Phi: x \to \sum_{l=1}^L N_{k,l}(t)
	d\left(\frac{S_{k,l}(t)}{N_{k,l}(t)},\kappa_l x\right)$ is convex and
    non-decreasing on $[\theta_k^{min}(t),1]$; the convexity is easily
    checked and $\theta_k^{min}(t)$ is defined as the minimum of this convex
	function. By definition, we have, either, $U_k(t,\delta)=1$ and then
	$U_k(t,\delta)>\theta_k$, or, $U_k(t,\delta)<1$ and $\Phi(U_k(t,\delta))= 
	\delta $, consequently
	\[\mathbb{P}\left(U_k(t,\delta) < \theta_k\right) 
	= \mathbb{P}\left(\Phi(U_k(t,\delta)) \leq \Phi(\theta_k) \right) = 
	\mathbb{P}\left(\delta \leq \Phi(\theta_k) \right).
	\]
\end{proof}

Remember that $U_k(t) = U_k(t, (1 + \epsilon) \log(t)) = U_k(t, f(t,\epsilon))$. Thus,
applying Proposition \ref{prop:PIEcontrol}, we obtain for arm $k$,

\begin{align*}
    \mathds{E}[|C_k|] \leq \sum_{t=1}^\infty \mathds{P}(U_k(t) \leq \theta_k)
        \leq \lceil e^{L+1} \rceil + \frac{e^{L+1}}{L^L}
        \sum_{t=\lceil e^{L+1} \rceil + 1}^\infty \frac{(2+\epsilon)^{2L}
        (\log t)^{3L}}{t^{1+\epsilon}} \leq C_3(\epsilon) ,
\end{align*}
for some constant $C_3(\epsilon)$.

\paragraph*{Controlling $\mathds{E}[|D|]$:}
Decompose $D$ as $D = \bigcup_{k=1}^L D_k$ where
$$D_k = \{t \geq 1 : t \in A \setminus (B\cup C), k \notin \mathcal{L}(t),
|\hat{\theta}_k(t) - \theta_k| \geq \eta \} .$$

For a given $k \leq L$, $D_k$ is the set of rounds at which $k$ is not one of
the leaders, and is not accurately estimated. Let $t \in D_k$. Since
$k \notin \mathcal{L}(t)$, we must have $\mathcal{L}_L(t) > L$. In turn, since
$t \notin B$, we have
$|\hat{\theta}_{\mathcal{L}_L(t)}(t) - \theta_{\mathcal{L}_L(t)}| \leq \eta$,
so that
$$\hat{\theta}_{\mathcal{L}_L(t)} \leq \theta_{\mathcal{L}_L(t)} + \eta \leq
\theta_L + \eta \leq (\theta_L + \theta_{L+1}) / 2 .$$

Furthermore, since $t \notin C$ and $1 \leq k \leq L$, we have $U_k(t) \geq
\theta_k \geq \theta_L \geq (\theta_L + \theta_{L+1}) / 2 \geq
\hat{\theta}_{\mathcal{L}_L(t)}$. This implies that $k \in \mathcal{B}(t)$
thus $\mathds{E}[k\in A(t)|t\in D_k] \geq 1/(2K)$. We apply 
Lemma \ref{lem:pie} with $H\equiv D_k$ and $c=1/(2K)$ to get
$$\mathds{E}[|D|] \leq \sum_{k=1}^L \mathds{E}[|D_k|] \leq 
4K(4K+\kappa_L^{-2}\eta^{-2}).$$

\subsection{Regret decomposition}

We decompose the regret by distinguishing rounds in $A \cup B$ and other rounds.
More specifically, we introduce the following sets of rounds for arm $k > L$:
\begin{equation*}
    E_k = \{t \geq 1 : t \notin (B \cup C \cup D), \mathcal{L}(t) = a^*,
             A(t) = v_{k,L}\} .
\end{equation*}
The set of instants at which a suboptimal action is selected now can be
expressed as follows
$$\{t \geq 1 : A(t) \neq a^*\} \subset (B \cup C \cup D)
\cup (\cup_{k=L+1} E_k) .$$
Using a union bound, we obtain the upper bound
$$\mathds{E}[R(T)] \leq \left(\sum_{l=1}^L \kappa_l\right) \mathds{E}[|B \cup C \cup D|] +
\sum_{k=L+1}^K \Delta_{v_{k,L}}(\theta) \mathds{E}[|E_k|] .$$

From previous boundaries, putting it all together, there exist $C_1(\eta)$ and
$C_3(\epsilon)$, such that
\[\left(\sum_{l=1}^L \kappa_l\right) (\mathds{E}[|B|] + \mathds{E}[|C|]
+ \mathds{E}[|D|]) \leq C_1(\eta) + C_3(\epsilon).\]

At this step, it suffices to bound events $E_k$ for all $k > L$.

\subsection{Bounding event $E_k$}

We proceed similarly to \cite{garivier11}.
Let us fix an arm $k > L$. Let $t \in E_k$: arm $k$ is pulled in
position $L$, so by construction of the algorithm, we have that $k \in
\mathcal{B}(t)$ and thus $U_k(t) \geq \hat{\theta}_{\mathcal{L}_L(t)}(t)$.
We first show that this implies that $U_k(t) \geq \theta_L - \eta$. Since
$t \in E_k$, we know that $\mathcal{L}_L(t) = L$, and since $t \notin B$,
$|\hat{\theta}_L(t) - \theta_L| \leq \eta$. This leads to
$$U_k(t) \geq \hat{\theta}_{\mathcal{L}_L(t)}(t) = \hat{\theta}_L(t)
\geq \theta_L - \eta .$$
Recall that $N_{k,L}(t)$ is the number of times arm $k$ was
played in position $L$. By denoting $d^+(x,y) = \mathds{1}\{x < y\} d(x,y)$,
we have that
\begin{align*}
    N_{k,L}(t) d^+(S_{k,L}(t)/N_{k,L}(t), \kappa_L(\theta_L-\eta)) &\leq N_{k,L}(t)
        d^+(S_{k,L}(t)/N_{k,L}(t), \kappa_L U_k(t)) \\
    &\leq \sum_{l=1}^L N_{k,l}(t)
        d^+(S_{k,l}(t)/N_{k,l}(t), \kappa_l U_k(t)) \leq f(t, \epsilon).
\end{align*}
This implies that $\mathds{1}\{t \in E_k\} \leq \mathds{1}\{N_{k,L}(t)
d^+(S_{k,L}(t)/N_{k,L}(t), \kappa_L(\theta_L - \eta)) \leq f(t, \epsilon)\}$.


\begin{lemma}\label{lem:garivier7}(\cite{garivier11}, Lemma $7$) Denoting by
    $\hat{\nu}_{k,s}^L$ the empirical mean of the first $s$ samples of $Z_{k,L}$, we
    have
    \begin{align*}
        \sum_{t=1}^T \mathds{1}\{A(t)=v_{k,L}, N_{k,L}(t)
        d^+(&S_{k,L}(t)/N_{k,L}(t), \kappa_L(\theta_L-\eta)) \leq f(t,\epsilon)\} \\
        &\leq \sum_{s=1}^T \mathds{1}\{sd^+(\hat{\nu}_{k,s}^L,
            \kappa_L(\theta_L-\eta)) \leq f(T,\epsilon)\} .
    \end{align*}
\end{lemma}

We apply Lemma \ref{lem:garivier7} which is a direct translation of Lemma $7$
from \cite{garivier11} to our problem. This yields
$$ |E_k| \leq \sum_{s=1}^T \mathds{1}\{sd^+(\hat{\nu}_{k,s}^L,
    \kappa_L(\theta_L-\eta)) \leq f(T,\epsilon)\} .$$

Let $\gamma > 0$. We define $K_T = \frac{(1+\gamma)f(T,\epsilon)}
{d^+(\kappa_L\theta_k, \kappa_L(\theta_L-\eta))}$. We now rewrite the last
inequality splitting the sum in two parts.

\begin{align*}
    \sum_{s=1}^T \mathds{P}(sd^+(\hat{\nu}_{k,s}^L,&\kappa_L(\theta_L-\eta))
    \leq f(T, \epsilon))
    \leq K_T + \sum_{s=K_T+1}^\infty \mathds{P}(K_T d^+(\hat{\nu}_{k,s}^L,
        \kappa_L(\theta_L-\eta)) \leq f(T,\epsilon)) \\
    &\leq K_T + \sum_{s=K_T+1}^\infty \mathds{P}(d^+(\hat{\nu}_{k,s}^L,
        \kappa_L(\theta_L-\eta)) \leq d(\kappa_L\theta_k,
        \kappa_L(\theta_L-\eta))/(1+\gamma)) \\
    &\leq K_T + \frac{C_2(\gamma,\eta)}{T^{\beta(\gamma,\eta)}} ,
\end{align*}
where last inequality comes from Lemma \ref{lem:klucb}. Fixing
$\gamma < \epsilon$, we obtain the desired result, which concludes the proof.

\begin{lemma}\label{lem:klucb}
    For each $\gamma > 0$, there exists $C_2(\gamma, \eta) > 0$ and
    $\beta(\gamma, \eta) > 0$ such that
    $$\sum_{s=K_T+1}^\infty \mathds{P}\left(d^+(\hat{\nu}_{k,s}^L,
        \kappa_L(\theta_L-\eta)) \leq \frac{d(\kappa_L\theta_k,
        \kappa_L(\theta_L-\eta)}{1+\gamma}\right) \leq
        \frac{C_2(\gamma,\eta)}{T^{\beta(\gamma,\eta)}} .$$
\end{lemma}

\begin{proof}
    If $d^+(\hat{\nu}_{k,s}^L, \kappa_L(\theta_L-\eta)) \leq
    \frac{d(\kappa_L\theta_k, \kappa_L(\theta_L-\eta))}{1+\gamma}$, then
    there exists some $r(\gamma,\eta) \in (\theta_k,\theta_L-\eta)$ such
    that $\hat{\nu}_{k,s}^L > \kappa_L r(\gamma,\eta)$ and
    $$d(\kappa_L r(\gamma,\eta),\kappa_L(\theta_L-\eta)) =
        \frac{d(\kappa_L\theta_k,\kappa_L(\theta_L-\eta))}{1+\gamma} .$$
    Hence,
    \begin{align*}
        \mathds{P}\left(d^+(\hat{\nu}_{k,s},\kappa_L\theta_L) <
            \frac{d(\kappa_L\theta_k,\kappa_L\theta_L)}{1+\gamma}\right) &\leq
            \mathds{P}\left(d(\hat{\nu}_{k,s},\kappa_L\theta_k) >
            d(\kappa_L r(\gamma,\eta),\kappa_L\theta_k), \hat{\nu}_{k,s}
            > \kappa_L\theta_k\right) \\
            &\leq \mathds{P}(\hat{\nu}_{k,s} > \kappa_L r(\gamma,\eta))
            \leq \exp(-sd(\kappa_Lr(\gamma,\eta),\kappa_L\theta_k)) .
    \end{align*}
    We obtain,
    $$\sum_{t=K_T}^\infty \mathds{P}\left(d^+(\hat{\nu}_{k,s},\kappa_L\theta_L) <
            \frac{d(\kappa_L\theta_k,\kappa_L\theta_L)}{1+\gamma}\right)
            \leq \frac{\exp(-K_Td(\kappa_Lr(\gamma,\eta),\kappa_L\theta_k))}
            {1 - \exp(-d(\kappa_Lr(\gamma,\eta),\kappa_L\theta_k))} \leq
            \frac{C_2(\gamma,\eta)}{T^{\beta(\gamma,\eta)}} ,$$
    for well chosen $C_2(\gamma,\eta)$ and $\beta(\gamma,\eta)$.
\end{proof}

\section{Lemmas}
In this section, we recall two necessary concentration lemmas directly adapted
from Lemma 4 and 5 in Appendix A of \cite{combes2015learning}. Although more
involved from a probabilistic point of view, these results are simpler to
establish than proposition~\ref{prop:boundCensUCB} as their adaptation to the
case of the PBM relies on a crude lower bound for $\tilde{N}_k(t)$, which is
sufficient for proving Theorem~\ref{th:pbm-pie}..

\begin{lemma} \label{lem:4} For $k\in\{1,\dots,K\}$ consider the martingale
  $M_{k,n} = \sum_{i=1}^{n}\bar{Z}_{k,i}$, where $\bar{Z}_{k,i}$ is defined
  in~\eqref{eq:Zbar}. Consider $\Phi$ a stopping time such that either
  $N_k(\Phi)\geq s$ or
  $\Phi = T+1$. Then
	\begin{align}
	\mathds{P}[|M_{k,N_k(\Phi)} | \geq N_k(\Phi) \eta, N_k(\Phi)\geq s ]\leq 
    2\exp(-2s\eta^2). \label{eq:unnormalized}
	\end{align}
	As a consequence,
	\begin{align}	
	 \mathds{P}[|\hat{\theta}_k(\Phi) -\theta_k | \geq  \eta,\, 
	\Phi\leq T] \leq 2\exp(-2s\kappa_L^2\eta^2).
	\label{eq:normalized}
	\end{align}
\end{lemma}
    
\begin{proof}
  The first result is a direct application of Lemma 4 of
  \cite{combes2015learning} as $(Z_l(t))_t$ with $Z_l(t) = X_l(t)Y_l(t)$ is an independent
  sequence of $[0,1]$-valued variables.

 For the second inequality, we use the fact that 
 $\tilde{N}_k(t) \geq \kappa_L N_k(t)$.
 Hence, 
 \[ \mathds{P}[|\hat{\theta}_k(\Phi) -\theta_k | \geq  \eta,\, 
 \Phi\leq T] \leq \mathds{P} \left[ \frac{|M_{k,N_k(\Phi)} |}{\kappa_L N_k(\Phi)} 
 \geq 
 \eta,\, \Phi\leq T \right] . \]
 which is upper bounded using (\ref{eq:unnormalized}).
\end{proof}

\begin{lemma}\label{lem:pie}
	Fix $c > 0$ and $k \in \{1,\dots,K\}$. Consider a random set of rounds
    $H \subset \mathbb{N}$, such that, for all $t$, $\mathds{1}\{t \in H\}$ is 
	$\mathcal{F}_{t-1}$ measurable and such that for all $t\in H$, $\{k\in
    \mathcal{B}(t) \}$ is true. Further assume, for all $t$, one has 
    $\mathds{E}[\mathds{1}\{k \in A(t) \}| t \in H]
	\geq c > 0$. We define $\tau_s$ a stopping time such that 
	$\sum_{t=1}^{\tau_s}
	\mathds{1}\{t \in H\} \geq s$. Consider the random set $\Lambda = 
	\{\tau_s : s \geq 1\}$.
	Then, for all $k$,
	\[ \sum_{t\geq 0} \mathds{P}[t \in \Lambda, |\hat{\theta}_k(t) - 
	\theta_k| \geq
	\eta] \leq 2 c^{-1}(2c^{-1} + \kappa_L^{-2}\eta^{-2}) \]
\end{lemma}   

The proof of this lemma follows that of Lemma 5 in \cite{combes2015learning} using the same lower bound for $\tilde{N}_k(t)$ as above.

\end{document}